%% file: main.tex
\documentclass[preprint]{article}
\usepackage{neurips_2026}
\usepackage{amsmath}
\usepackage{amssymb,amsfonts,amsthm}
\usepackage{mathrsfs}
\usepackage{mathtools}
\usepackage{algorithm}
\usepackage{algpseudocode}
\usepackage{xcolor}
\usepackage{dsfont}
\usepackage{thmtools}
\usepackage{thm-restate}
\usepackage{enumitem}
\usepackage{dsfont}
\usepackage{graphicx}
\usepackage{subcaption}
\usepackage{sidecap}
\usepackage{accents}
\usepackage{wrapfig}
\usepackage{booktabs} 
\usepackage{longtable}
\usepackage{siunitx}
\usepackage{comment}
\newtheorem{definition}{Definition}[]
\newtheorem{theorem}{Theorem}[section]

\newtheorem{corollary}{Corollary}[]
\newtheorem{proposition}{Proposition}[]

\newtheorem{problem}{Problem}[]
\newtheorem{fact}{Fact}[]
\newtheorem{lemma}{Lemma}[]
\newtheorem{remark}{Remark}[]

\usepackage{hyperref}
\usepackage{xcolor}
\hypersetup{
  colorlinks   = true, 
  urlcolor     = blue, 
  linkcolor    = blue, 
  citecolor   = blue 
}
\usepackage[capitalize, nameinlink]{cleveref}

\usepackage[utf8]{inputenc} 
\usepackage[T1]{fontenc}    
\usepackage{hyperref}       
\usepackage{url}            
\usepackage{booktabs}       
\usepackage{amsfonts}       
\usepackage{nicefrac}       
\usepackage{microtype}      
\usepackage{xcolor}

\newcommand{\maxmargin}{\mathsf{m}^{rd}}

\newcommand{\signrank}{\mathsf{sr}}
\newcommand{\dsignrank}{\mathsf{dsr}}
\newcommand{\sigloss}{\mathcal{L}_{\text{sigmoid}}}
\newcommand{\infonceloss}{\mathcal{L}_{\text{InfoNCE}}}
\newcommand{\Betadist}{\mathsf{Beta}}
\newcommand{\prob}{\mathbb{P}}
\newcommand{\expect}{\mathbb{E}}

\newcommand{\tV}{\widetilde{V}}
\newcommand{\tU}{\widetilde{U}}
\newcommand{\vectorize}{\mathsf{vec}}
\newcommand{\ksparsen}{S_{n,k}}
\newcommand{\conv}{\mathsf{conv}}

\title{Is Dimensionality a Barrier for Retrieval Models?}
\author{%
  Kiril Bangachev\thanks{Author names appear in alphabetical order by last name.}\thanks{Supported by NAE Grand Challenge Vest Fellowship.} \\
  \texttt{kirilb@mit.edu} \\
  \And 
  Guy Bresler\thanks{Supported by NSF Award 2428619.}\\
  \texttt{guy@mit.edu}\\
  \And
  Jonathan Kogan\\ \texttt{jkogan1@mit.edu} \\
  \And
  Yury Polyanskiy\thanks{Supported by NSF Grant CCF-21-12665 and by generous gifts from Jane Street and Google Research.}\\
  \texttt{yp@mit.edu}\\ \\
  Department of Electrical Engineering and Computer Science\\
  Massachusetts Institute of Technology\\
  Cambridge, MA, 02139 \\}

\date{May 2025}

\begin{document}

\maketitle

\begin{abstract}
\textit{Why does the low dimensionality of representations, typically $d\approx 1000$, not prevent modern embedding-based retrieval models from scaling to billions, or even trillions, of data points?}
To answer this question, we study maximal-margin embeddings in the following retrieval model, classically studied in communication complexity \cite{paturi1986probabilistic} and more recently in embedding-based retrieval \cite{weller2025theoretical}.
Let $A\in \{0,1\}^{N\times n}$ be a matrix indicating whether each of $N$ queries is relevant to each of $n$ documents. We are interested in the largest \textit{margin} $m>0,$ denoted by $\mathsf{m}^{\mathsf{rd}}(d, A),$ for which there exist unit norm embeddings of the queries and documents $\{U_j\}_{j = 1}^N, \{V_i\}_{i = 1}^n$ with the following property. $\langle U_j, V_i\rangle \ge m$ whenever $A_{ji} = 1$ and $\langle U_j, V_i\rangle \le -m$ otherwise. A large margin is a key proxy for representation quality: it controls both robustness to perturbations and compositional generalization across queries. Our main theorem establishes that the best possible margin without a restriction on the dimension, $\mathsf{m}^{\mathsf{rd}}(+\infty, A),$ can be nearly achieved in dimension $d = O(\mathsf{m}^{\mathsf{rd}}(+\infty, A)^{-2}\log n)$ which improves a theorem of \cite{ben2002limitations}. Together with a matching lower bound in Theorem 1.5, we conclude that when $A\in \{0,1\}^{\binom{n}{k}\times n}$ is the matrix containing all possible $k$-sparse rows once, dimension $d = O(k\log (n/k))$ is necessary and sufficient for the maximal possible margin $\mathsf{m}^{\mathsf{rd}}(+\infty, A) = \Theta(k^{-1/2})$ in this setting. This fully resolves the setup of \cite{weller2025theoretical}. We also give several constructions for large margins when $d = o(k\log (n/k)).$ Our proofs are based on connections between this problem and the literature on compressed sensing and the restricted isometry property. Finally, we empirically test the InfoNCE and sigmoid losses for producing large margin embeddings and demonstrate a clear advantage of the sigmoid loss. Code is available at \href{https://github.com/BangachevKiril/TopK}{github/TopK}.
\end{abstract}

\begin{figure}[htb!]
    \centering
    \begin{minipage}[t]{0.47\textwidth}
        \centering
        \includegraphics[width=.9\linewidth]{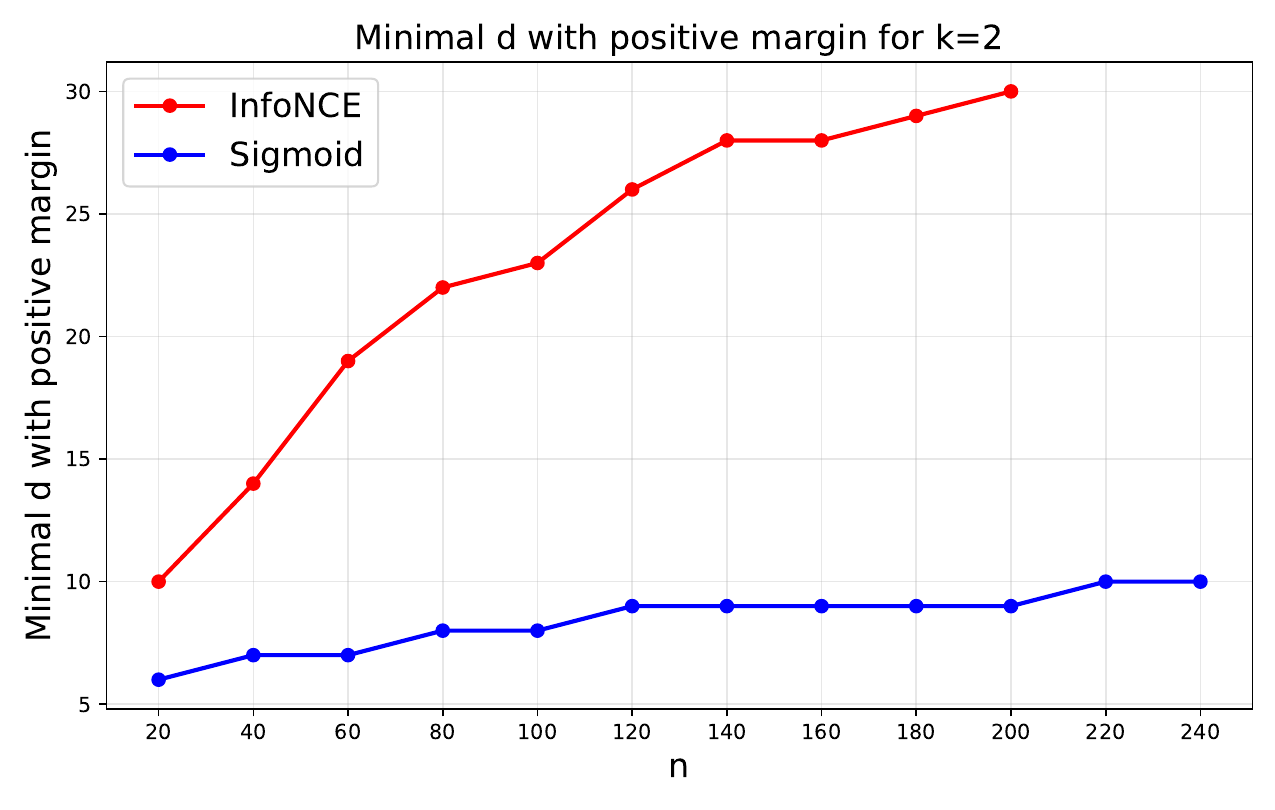}
        \caption{\small{Minimal dimension needed to achieve a non-zero margin after 100000 training steps for the InfoNCE and sigmoid losses. The sigmoid succeeds in much smaller dimensions.}}
        \label{fig:min-nonzero-margin}
    \end{minipage}
    \hfill
    \begin{minipage}[t]{0.51\textwidth}
        \centering
        \includegraphics[width=.9\linewidth]{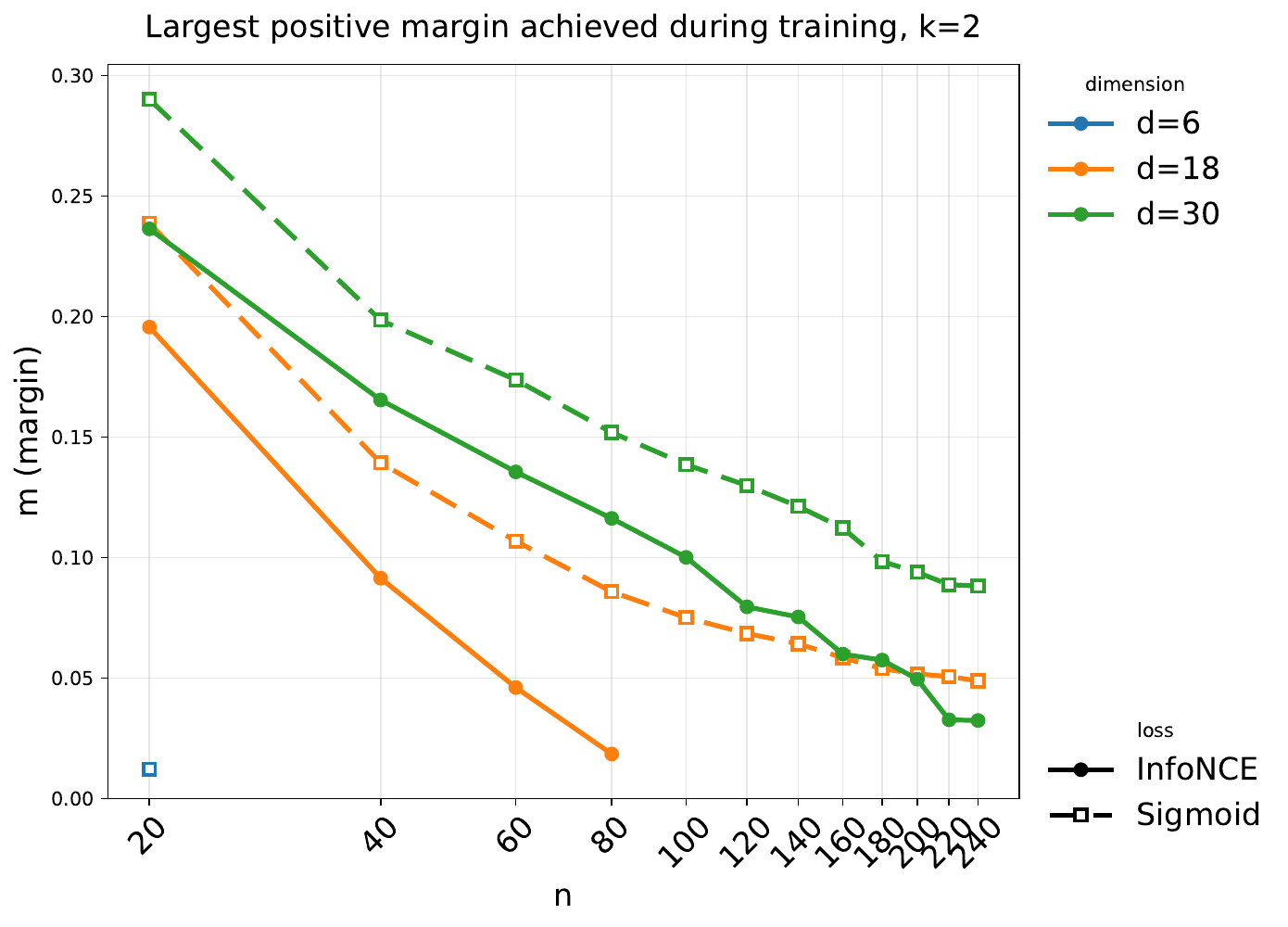}
        \caption{\small{Maximal margin achieved during 100000 training steps for $d \in \{6,18,30\}$, using the InfoNCE and sigmoid losses showing clear advantage of the sigmoid loss.}}
        \label{fig:max-margin-vs-logn}
    \end{minipage}
\end{figure}

\input{Introduction}
\input{JL}

\input{LowerBoundsv1}

\input{Construction}

\input{Limitations}

\section*{Acknowledgements}
KB is thankful to Pravesh Kothari and Sidhanth Mohanty for insightful conversations at different stages of the current work.

\newpage
\bibliography{ref}
\bibliographystyle{alpha}
\appendix

\input{SimpleAppendix}

\input{AppendixMissingConstruction}

\input{MarginAlgebraic}

\input{SpectralAndSnk}

\input{LargeMarginFromSparseCoding}

\input{LowerBoundBeyondnchoosek}

\input{ProofsofMarginImportance}

\input{FreeEmbedding}

\input{LLMs}

\newpage

\end{document}

%% file: Introduction.tex
\section{Introduction}\label{sec:intro}

\subsection{Background and Motivation}\label{subsec:background}
Modern information retrieval systems rely on \emph{dense neural embeddings of data}.
Depending on the modality, a variety of training paradigms and architectures exist for representing data as vectors: for text, popular choices include transformer-based architectures trained with contrastive or self-supervised objectives over query-passage pairs such as SBERT, DPR, E5, GTE, BGE, EmbeddingGemma \cite{reimers2019sentence,KarpukhinEtAl2020,wang2022e5,li2023gte,chen2024bgem3,vera2025embeddinggemma}; for images, popular choices include vision transformers trained with self-distillation or reconstruction objectives such as the DINO family and ViT-MAE \cite{caron21dino, OquabEtAl2023,he2022mae,simeoni2025dinov3}; finally, in multimodal settings contrastive models such as CLIP and SigLIP are common~\cite{radford21clip,zhai23siglip,tschannen2025siglip2}.

Once data is represented compactly in vector form (typically $d\le 4096$ across the wide range of encoders), retrieval is often done via a simple nearest neighbor search around the query embedding \cite{jegou2011product,johnson2019billion,malkov2020hnsw,radford21clip,caron21dino,chen2020simple,dwibedi2021little,bolya2025perception} (and many others). The \emph{compactness of the embeddings} enables low-memory caching of representations and fast algorithmic primitives (such as nearest neighbor search). Both speed and low-cost storage are integral to high-performance  AI. However, compactness of representations is not enough on its own. 
For the design of accurate, reliable, and precise AI systems, it is also necessary that embeddings be of \emph{high-quality}.

That \emph{compactness} and \emph{quality} are at odds is no surprise. Low dimension (rank) limits the geometric configurations available to item representations. What should be surprising instead is that models based on low-dimensional representations still achieve excellent performance on retrieval from huge datasets.
For instance, DPR embeds data in $d = 768$ dimensions and achieves remarkable performance when retrieving from a pool of $\approx 2\times 10^7$ passages \cite{KarpukhinEtAl2020}. This raises the question:

\begin{quote}
\centering
\emph{Q: How can embedding models be so successful despite their low dimensionality?}
\end{quote}

We approach this question theoretically by analyzing the top-$k$ retrieval model recently introduced in \cite{weller2025theoretical}. By studying a key quantity governing the quality of embeddings -- \emph{the margin} -- we show a surprising low-rank saturation result (formally stated in Theorem~\ref{thm:main}): 
\begin{quote}
\centering
\emph{Informal Main Result: The highest possible quality embeddings are already achieved in a dimension growing only logarithmically with the data size.}
\end{quote}
Hence, embedding-based retrieval models do not require a high dimension for optimal performance. This main result opens the door to 
further theoretical results related to precise dimension recommendations for encoders and linking different notions of ``quality'' of representations. To obtain these results, we connect the model of \cite{weller2025theoretical} to the classical literature on compressed sensing and the restricted isometry property \cite{CandesTao2005,foucart2010gelfand,foucart2013mathematical,novak1995optimal}, communication complexity \cite{AlonFranklRodl1985,sherstov2009lower,linial2009learning}, and learning theory \cite{ForsterSchmittSimonSuttorp2003}.   

\subsection{Mathematical Model and Problem Formulation}
\label{sec:model}

\textbf{Basic Retrieval Model.} In the model of \cite{weller2025theoretical}, an incidence matrix $A\in \{0,1\}^{N \times n}$ encodes the relevance relationship between $N$ queries and $n$ objects. A query $j \in \{1,2,\ldots, N\} = [N]$ is relevant to object 
$i \in [n]$ if and only if $A_{ji} = 1.$ This abstraction captures a variety of practical settings, for example when $[n]$ corresponds to a corpus of passages and $[N]$ to a set of queries \cite{KarpukhinEtAl2020}, or  when $[N]$ corresponds to the images in a dataset and $[n]$ to their classes \cite{russakovsky2015imagenet}. 

The results of \cite{weller2025theoretical} focus primarily on the case where $N = \binom{n}{k}$ and $A=\ksparsen\in \{0,1\}^{N \times n}$ contains every possible $k$-sparse row exactly once. In this setting, for every group of $k$ objects, there is a query to which exactly those $k$ objects are relevant. 
Many of our results apply more broadly to 
matrices $A$ for which all rows are distinct and $k$-sparse (so that only some groups of $k$ documents can be retrieved, which is more realistic). Typically, we will think about the quantity $k$ as a large constant or slowly growing with $n$, even though our results apply for any $k$ unless otherwise stated. We note that \cite{bangachev2025globalminimizerssigmoidcontrastive} analyzes the case when $k = 1$,  $\ksparsen = I$.

Corresponding to $A$, models $f_\theta$ and $g_\phi$ produce embeddings, respectively $\{f_\theta(O_i)\}_{i = 1}^n = \{V_i\}_{i = 1}^n$ of the objects $\{O_i\}_{i = 1}^n\subseteq \mathcal{O}$ and $\{g_\phi(Q_j)\}_{j = 1}^N = \{U_j\}_{j = 1}^N$ of the queries $\{Q_j\}_{j = 1}^N\subseteq \mathcal{Q}$. These embeddings are assumed to be unit norm throughout the paper, which is without loss of generality as only cosine similarities between them will be considered. 

The goal of the representations is to enable retrieval of objects relevant to a given query. One property that ensures this is if there exists some $\tau\in \mathbb{R}$ such that:
\begin{equation}
\label{eq:zeromarginretrieval}
   \langle U_j, V_i\rangle >\tau \text{ if }A_{ji}=1\text{ and }\langle U_j, V_i\rangle <\tau \text{ if }A_{ji} = 0. 
\end{equation}
Indeed, this would imply that the documents relevant to query $j$ are those whose embeddings have cosine similarity with $U_j$ exceeding $\tau$. These documents can be retrieved by the standard procedure of finding the nearest neighbors of $U_j$ and filtering out those with cosine similarity less than $\tau.$

\textbf{{Retrieval Margin.}} While property \eqref{eq:zeromarginretrieval} guarantees correct retrieval, it is too brittle to be useful in practice. A minor modification $O'_i$ of object $O_i,$ for example flipping few pixels, may lead to an incorrect retrieval by some query $Q_j,$ i.e. 
$\langle f_\theta(O_i), g_\phi(Q_j)\rangle <
\tau <
\langle f_\theta(O'_i), g_\phi(Q_j)\rangle.$ To capture the fact that strong models should succeed on retrieval even when data is slightly perturbed, we introduce the \emph{margin} of a model. Borrowing the terminology of \cite{bangachev2025globalminimizerssigmoidcontrastive}, we make the following definition.

\begin{definition}
\label{def:maindefinition}
Given an incidence matrix $A\in \{0,1\}^{N\times n},$ we say that the vectors $\{V_i\}_{i = 1}^n , \{U_j\}_{j = 1}^N$ are a margin-$m$, relative-bias-$\tau$ embedding of $A$ for $m\ge 0$ if
\begin{equation}
\label{eq:marginretrieval}
   \langle U_j, V_i\rangle \ge \tau +m \text{ if }A_{ji}=1\text{ and }\langle U_j, V_i\rangle \le \tau-m \text{ if }A_{ji} = 0. 
\end{equation}
\end{definition}

The main question that we will study is bounding the maximal possible margin of the embedding. We note that for the rest of the paper, we will only focus on the case of $\tau = 0$ for simplicity of exposition. Following \cite{weller2025theoretical}, this is without loss of generality up to increasing the dimension by 1 and multiplying the margin by a small absolute constant. See Appendix~\ref{appendix:somesimpleclaims} for completeness.
\begin{problem}
\label{problem:mainproblem}
For a given $A\in \{0,1\}^{N\times n},$ find the optimal restricted-dimension margin 
$$\maxmargin(d, A)
\coloneqq \sup \{ m \; : \; \exists \text{ margin-}m,\text{ relative-bias-}0\text{ embeddings of }A\text{ in dimension }d\}.
$$ 
Whenever there does not exist a margin-$m$ embedding of $A$ for any $m\ge 0,$ we set $\maxmargin (d, A) = -\infty.$ 
\end{problem}
The quantity without the dimension restriction, i.e. $\maxmargin(+\infty, A) \coloneqq \sup_d \maxmargin(d, A)$ is known as the \emph{margin complexity} in the communication complexity literature \cite{sherstov2009lower}. \footnote{Dimension $n$ is enough to obtain maximal margin, but we stick to the $+\infty$ notation to emphasize the absence of dimension restriction. Also, the $\sup$ can be replaced by $\max$ due to compactness.} The minimal $d$ for which an embedding exists is the \emph{sign-rank}, 
$\signrank(A)\coloneqq \min \{ d\; : \maxmargin(d, A)> 0\}$ \cite{AlonFranklRodl1985}.

\textbf{Significance of Large-Margin Embeddings.} Large-margin learning has been of interest for many years to the machine learning community \cite{novikoff1962convergence,cortes1995support,bartlett1998sample,freund1999large,cottersvm,zhu031normsvm,crammer02multiclasssvmlargemargin,smoladvancesinlargemargin,balcan2006kernels,arriaga2006algorithmic,bartlett2017spectrally}. Concrete works on representation learning which highlight the importance of large margins include 
\cite{balcan2008theory,
hadsell2006drlim,weinberger2009distance,schroff2015facenet,deng2019arcface,sokolic2017robust,elsayed2018large,robinson2021contrastive,
    bangachev2025globalminimizerssigmoidcontrastive}. We outline several reasons why it is concretely important for retrieval. The statements and proofs are simple and delayed to Appendix~\ref{sec:margin-importance}.
\begin{itemize}
    \item \emph{Large Margin and Robustness:} We already stated why small-margin embeddings are brittle. 
    We formalize this idea for Lipschitz encoders in Proposition~\ref{prop:marginrobustness}.
    \item \emph{Large Margin and Quantization:} In particular, as large margin embeddings are more robust, they allow for coarser quantization while retaining perfect retrieval. See Proposition \ref{prop:marginquantization}.
    \item \emph{Large Margin and Compositional Generalization:} In addition to generalization to correctly retrieving perturbed objects, larger margins also yield higher compositional generalization. Concretely, in the $\ksparsen$ model, suppose the object embeddings $\{V_i\}_{i = 1}^n$ are such that every subset of $k$ of them can be retrieved with margin $m$ via a query embedding. Then every subset of up to $k + 2mk$ objects can also be retrieved. Hence, a larger margin leads to a larger class of queries that can be correctly retrieved. See Proposition~\ref{prop:margincompositional}.
    \item \emph{Large Margin and Sigmoid Loss:} Following \cite{bangachev2025globalminimizerssigmoidcontrastive}, we show that the global minimizers of sigmoid contrastive loss (with positive pairs determined by the incidence matrix $A$) exactly coincide with margin-$m$, relative-bias-$\tau$ embeddings and the loss has 
    an inductive bias towards large-margin embeddings. See Proposition~\ref{prop:marginsigmoid} and \cite{soudry2018implicit} for related work.
\end{itemize}


\textbf{Some Preliminary Observations.} Note that if $B$ is a submatrix of $A,$ for any $d,$
$\maxmargin(d, B)\ge \maxmargin(d,A)$ (since one can just take the embeddings corresponding to the respective rows and columns). Hence, in particular, $\maxmargin(d,A) \ge \maxmargin(d, \ksparsen)$ for any $A$ with $k$-sparse rows as well. Similarly, $d\longrightarrow \maxmargin(d, A)$ is non-decreasing for any fixed $A.$
\subsection{Prior work}
\label{subsec:prior-work}

\textbf{Sign Rank.} In \cite{paturi1986probabilistic}, the authors  establish that for square matrices $A,$ $\log(\signrank(A))$ equals (up to additive error 1) the two-way communication complexity of computing $A$ in the model of Yao \cite{yaocc79}. In \cite{alon2016sign}, on the other hand, the authors show that a dual version of the sign rank is equivalent (up to a multiplicative factor of $2$) to the VC dimension of the function class corresponding to rows of $A.$ The following results from these two lines of work are relevant. We provide proofs in Appendix~\ref{sec:marginalgebraic} for completeness (as they are stated in a different language in the original papers). 

\begin{theorem}[ \cite{AlonFranklRodl1985,alon2016sign}]
\label{thm:signrankforsparse}
If every row of $A\in \{0,1\}^{N \times n}$ is $k$-sparse, then $\signrank(A)\le 2k+1$ \cite{AlonFranklRodl1985}.
Conversely, when $n \ge 3k+1,$ $\signrank(\ksparsen)\ge 2k+1$ \cite{alon2016sign}. 
\end{theorem}

In  
Appendix~\ref{sec:marginalgebraic}, we analyze the margin of the $\signrank(A)\le 2k+1$ construction to obtain Corollary~\ref{cor:marginalgebraic}.

\textbf{Margin Complexity.} It turns out that the margin complexity $\maxmargin(+\infty, A)$ equals (up to a multiplicative error of $8$) the discrepancy of $A$ \cite{linial2009learning}. Discrepancy is also used in communication complexity lower bounds \cite{yao83LB,babai86complexity,sherstov2009lower}. Important for us is the following spectral bound on margin complexity, which follows by combining the margin-complexity discrepancy equivalence of \cite{linial2009learning} and the spectral bound of \cite{nisancc95}. We give a more direct and simpler proof in Appendix~\ref{sec:spectral}.
\begin{theorem}[Spectral Bound on $\maxmargin(+\infty, A)$ \cite{forster2002linear,sherstov2009lower,nisancc95}]\label{thm:spectralboundonmargincomplexity}
For any $A\in \{0,1\}^{N\times n},$
\begin{align*}
& \maxmargin(+\infty, A) \le \inf \|J\|_{op}\sqrt{Nn} \text{ over }J\in \mathbb{R}^{N\times n}, \text{ s.t.,}\\
& J_{ji}\ge 0\;  \forall \; A_{ji} =1, 
J_{ji}\le 0\;  \forall \; A_{ji} =0, \text{ and } 
\sum_{j,i}|J_{ji}| = 1.
\end{align*}
\end{theorem}
This
upper bound on the margin complexity can be computed efficiently as it minimizes a convex objective over a convex space. In the case of $\ksparsen$ we can choose $J_{ji} = \frac{1}{2Nk}$ whenever $A_{ji} = 1$ and $J_{ji} = -\frac{1}{2N(n-k)}$ otherwise to conclude the $1/(2\sqrt{k})$ upper bound on $\maxmargin(+\infty, \ksparsen).$ A simple construction shows that this is tight. Both upper bound and construction are in Appendix~\ref{sec:spectral}.
\begin{corollary}
\label{cor:sqrtkmargin} For $k = o(n),$
    $\maxmargin(+\infty, \ksparsen)= \frac{1+o_k(1)}{2\sqrt{k}}.$
\end{corollary}
Distributional versions of margin complexity have also been studied \cite{kamath2020approximate}.

\textbf{Embedding-Based Retrieval.} While \cite{weller2025theoretical} and the follow-up work~\cite{WangEtAl2026} mostly focus on sign rank (existence of any positive margin), the following result is relevant. 

\begin{theorem}[\cite{weller2025theoretical}]
\label{thm:spherepackinglbweller}
For any dimension-$d$ margin-$m$ embedding of $\ksparsen,$ it holds that
    $d\ge \frac{\log \binom{n}{k}}{\log(1 + 1/m)}.$ Rewriting, $\maxmargin(d,\ksparsen)\le \big({\binom{n}{k}^{1/d} - 1}\big)^{-1}.$
\end{theorem}
While a valuable first step, this result falls short of fully answering our motivating question and Problem~\ref{problem:mainproblem}. First, it is a negative result, giving only an upper bound on the achievable margin. \cite{weller2025theoretical} does not provide a complementary positive result demonstrating whether the bound is tight (as we will see shortly, the bound is \textbf{not} tight). In particular, it does not elucidate why low-dimensional retrieval models are so successful in practice.     

\textbf{(Random) Spherical Graphs.} Finally, we note that Definition~\ref{def:maindefinition} is related to bipartite masks of spherical geometric graphs. Concretely, a spherical graph $G$ on vertex set $[p]$ is defined by unit vectors $X_1, X_2, \ldots, X_p\in \mathbb{S}^{d-1}$ associated to the vertices. Two vertices $a,b$ are adjacent whenever $\langle X_a, X_b\rangle \ge \tau$ for some fixed $\tau.$ In our model, $p = n+N,$ the vertices form a bipartition $[N]\cup \{N+1, \ldots, N+n\},$ the latent vectors are $\{U_j\}_{j = 1}^N, \{V_i\}_{i = 1}^n,$ and considered are only edges between the two parts $[N], \{N+1, \ldots, N+n\}.$ This is exactly the bipartite mask setting in \cite{brennan2021finettistyleresultswishartmatrices}.
Recent work on spherical geometric graphs focuses on the setting when the vectors are i.i.d. uniformly random on the sphere \cite{bubeck2016testing,liu2023local,ma2025exponential}. We want to emphasize the paper \cite{gamarnik2020explicit} which takes a rather different direction.  In this work, the authors show how to construct a geometric graph satisfying a strong Ramsey property from a matrix satisfying a strong restricted isometry property (RIP). This establishes yet another, unexpected, connection between RIP and our model.

\subsection{Results}\label{subsec:main-results}
We highlight a gap in the literature. There is no positive result which guarantees simultaneously \emph{large margin} (quality) and a \emph{small dimensionality} (compactness) of representations.

\textbf{Part 1: Main Result.}
Our main result directly addresses this gap. 

\begin{theorem}[Main Result] 
\label{thm:main}
Let $A\in \{0,1\}^{N\times n}.$
For every $\epsilon>0,$ there exists some $C_\epsilon = O(\epsilon^{-4})$ depending only on $\epsilon$ such that 
$$
\maxmargin(C_\epsilon \; \times \; \maxmargin(+\infty, A)^{-2}\; \times \; \log(\min(n , N))\; , \; A)\ge (1-\epsilon)\maxmargin(+\infty, A).
$$
\end{theorem}
We interpret this result as follows: one can obtain nearly the optimal margin $\maxmargin(+\infty, A)$ (i.e., nearly the same as in dimension $d=\infty$) in dimension $O(\maxmargin(+\infty, A)^{-2}\; \times \; \log(\min(n , N))).$ For concreteness, recall that if each row of $A$ is $k$-sparse, then $\maxmargin(+\infty, A) \ge \frac{1-o_k(1)}{2\sqrt{k}}.$ In particular, an optimal margin for any $k$-sparse matrix can be obtained in dimension $O(k\log n)$!

It turns out that one can recover that $\maxmargin(Ck\log n, \ksparsen) = \Omega(1/\sqrt{k})$ for a large enough $C$ from results on compressed sensing \cite{CandesTao2005}. However, it seems like the approach is limited to 
margin $\Theta(1/\sqrt{k})$ and dimension $\Theta(k\log n)$ and falls short of our Theorem~\ref{thm:main}.  Theorem~\ref{thm:main} goes further and implies, for example, that for $k$-sparse $A$ with $\maxmargin(+\infty, A) = \omega_k(k^{-1/2}),$ maximal margin can be achieved in an even smaller dimension $o_k(k\log n).$ 
See Appendix~\ref{sec:largemarginviasparsecoding}.

We note that Theorem~\ref{thm:main} is a substantial strengthening of the result from \cite{ben2002limitations} that states that $\signrank(A) = O(\maxmargin(+\infty, A)^{-2}\log(nN)).$ The result of \cite{ben2002limitations} is for sign rank, so it only states that a positive margin embedding exists, without preserving the margin. More importantly, our result has $\log(\min(n , N))$ dependence instead of 
$\log(nN) = \Theta(\log(\max(n,N))).$ Instead of our $O(k\log(n)),$ \cite{ben2002limitations} achieves dimension $O(k^2\log(n))$ for margin $\Theta(k^{-1/2})$ for $\ksparsen.$

Both the proof of our Theorem~\ref{thm:main}, presented in Section~\ref{sec:mainthmproof}, and the result of \cite{ben2002limitations} use the Johnson-Lindenstrauss lemma (Theorem~\ref{thm:JL}) to perform dimensionality reduction. Concretely, the result $\signrank(A) = O(\maxmargin(+\infty, A)^{-2}\log(nN))$ proceeds by taking embeddings $\{U_j\}_{j = 1}^N, \{V_i\}_{i = 1}^n$ attaining the maximal margin $\maxmargin(+\infty, A)$
and applying JL on the $Nn$ pairs $(U_j, V_i)$ so that inner products are preserved up to an additive error smaller than the margin $\maxmargin(+\infty, A).$ 

The novelty in our proof is that we apply JL so that \emph{only the norms} of (certain convex combinations of the) $\{V_i\}_{i = 1}^n$ (when $n\le N$) are approximately preserved. Then, we use the following non-constructive certificate for the existence of  the desired $N$ vectors corresponding to queries instead of explicitly constructing them. The proof of Lemma~\ref{lem:convexhullequivalentofmargin} is simple and delayed to Appendix~\ref{appendix:somesimpleclaims}. 

\begin{lemma}
\label{lem:convexhullequivalentofmargin}
Given are vectors $\{V_i\}_{i = 1}^n\in \mathbb{R}^D,$ a sign vector $\sigma\in \{\pm1\}^n,$ and some $m>0.$ Then, 
\begin{enumerate}
    \item [$(\Longrightarrow)$] Suppose that there exists a unit vector $U$ with the following property. Whenever $\sigma_i = 1,$ it holds that $\langle U, V_i\rangle \ge m$ and whenever $\sigma_i = -1,$ it holds that $\langle U, V_i\rangle \le -m.$ Then, for every $x\in \conv(\{\sigma_i V_i\}_{i = 1}^n),$ it holds that $\|x\|_2\ge m.$
    \item [$(\Longleftarrow)$] Conversely, suppose that $m = \min\big\{\|x\|_2\; : \; x \in \conv(\{\sigma_i V_i\}_{i = 1}^n)\big\}.$ Then, there exists a unit vector $U$ with the following property. Whenever $\sigma_i = 1,$ it holds that $\langle U, V_i\rangle \ge m$ and whenever $\sigma_i = -1,$ it holds that $\langle U, V_i\rangle \le -m.$
\end{enumerate}
\end{lemma}

We show that this property is satisfied by the (convex combinations of) $\{V_i\}_{i = 1}^n$ after the JL-projection via Maurey's Approximation Lemma (Theorem~\ref{thm:Murrayslemma}). The idea for using Maurey's approximation lemma originates from the literature on compressed sensing \cite{foucart2010gelfand,foucart2013mathematical,novak1995optimal}. 

\textbf{Part 2: Impossibility Results.} We next prove that the dimension bound in Theorem~\ref{thm:main} is tight for the $\ksparsen$ model. In the maximal margin regime $m = \Theta(1/\sqrt{k})$, Theorem~\ref{thm:spherepackinglbweller} of \cite{weller2025theoretical} gives the lower bound $d=\Omega(k\log(n/k)/\log(k))$ while Theorem~\ref{thm:main} gives an upper bound of $d = O(k\log(n))$. There is still a $\log(k)$ gap between these upper and lower bounds. We close this gap with the following theorem, which we prove in Section~\ref{sec:Lower Bounds}:

\begin{theorem}
\label{thm:mainlowerbound} Suppose that $n\ge 3k$ and 
$\maxmargin(d, \ksparsen)>0.$
For some universal constant $C>0,$
$$
d\ge C\frac{k\log (n/k)}{\log(1 + 2(\maxmargin(d, \ksparsen)\sqrt{k})^{-1})}.
$$
\end{theorem}
We conclude that the minimal dimension required to obtain margin $m = \Theta(1/\sqrt{k})$ is exactly $\Theta(k\log(n/k)),$ thus fully resolving the $\ksparsen$ case \cite{weller2025theoretical}. Theorem~\ref{thm:mainlowerbound} can be adapted to general $k$-sparse $A,$ but the resulting bound is more complex. See Appendix~\ref{sec:LowerBoundBeyondnchoosek}.

We briefly compare the proof of our Theorem~\ref{thm:mainlowerbound} with the proof of the weaker $d=\Omega(k\log(n/k)/\log(1/m))$ lower bound of \cite{weller2025theoretical}. In \cite{weller2025theoretical}, the authors simply use the fact that if $\{U_j\}_{j = 1}^N, \{V_i\}_{i = 1}^n$ are a margin $m$ embedding of $\ksparsen,$ then the vectors $\{U_j\}_{j = 1}^N$ form a radius-$m$ packing of the sphere, to which they apply standard volume arguments. 

Our proof also relies on a volume argument. However, instead of directly applying it on the query embeddings $\{U_j\}_{j = 1}^N$ or document embeddings $\{V_i\}_{i = 1}^n,$ we pursue an idea akin to the proof of Theorem~\ref{thm:main}. We non-constructively show the existence of a collection of vectors which also has size $\exp(\Omega(k\log(n/k)))$ but is a packing with a much higher radius of $\Omega(m\sqrt{k}).$ This collection is formed by linear combinations of the vectors $\{V_i\}_{i = 1}^n$ and again crucially relies on Lemma~\ref{lem:convexhullequivalentofmargin}. 

\textbf{Part 3: Constructions beyond the maximal margin.} Theorems~\ref{thm:main} and \ref{thm:mainlowerbound} primarily address the question of embedding $k$-sparse matrices with maximal margin $\Theta(1/\sqrt{k}).$ However, other non-trivial margins could be possible in dimensions smaller than $\Theta(k\log n).$ For instance, a simple analysis of the $2k+1$ sign-rank construction of \cite{AlonFranklRodl1985}, provided in Appendix~\ref{sec:marginalgebraic}, yields the following result.

\begin{corollary}
\label{cor:marginalgebraic}
    $\maxmargin(2k+1, A)\ge (2n)^{-2k}(2k)^{-1/4}$ if all rows of $A$ are $k$-sparse.
\end{corollary}

We prove a separate result which establishes a smooth trade-off between dimension and margin. It is based on self-Khatri-Rao lifts \cite{khatri1968solutions}. Namely, we form the $V_i$ vectors roughly by drawing i.i.d. vectors $a_1, a_2, \ldots, a_n \in \mathbb{S}^{\sqrt{d}-1}$ and setting $V_i = a_i\otimes a_i\in \mathbb{S}^{d-1}.$ The improvement is also related to the aforementioned connection to the restricted isometry property. Khatri-Rao random lift matrices are known to satisfy stronger RIP conditions than random matrices \cite{khanna18khatrirao}. 

\begin{theorem}
\label{thm:main_construction}
Suppose that $A\in \{0,1\}^{N\times n}$ is such that every row has $k$ non-zero entries for some $k\ge 2.$ Suppose furthermore that
$d\ge \frac{k^2 + 5k + 7}{2}$ and let $r\coloneqq\lfloor \frac{\sqrt{8d-7}-1}{2}\rfloor = \sqrt{2d}+O(1).$ Then, 
$$
\maxmargin(d, A)\ge 
\frac{\Delta}{2(\sqrt{k} + 1) - \Delta}
$$    
where $\Delta = \Theta\Big( \frac{r-k}{r}\times \Big(\frac{1}{(n-k)N}\Big)^{\frac{2}{r - k}}\Big).$
\end{theorem}
The proof is in Section~\ref{sec:proofofconstruction}. We illustrate the bound with the following cases when $N = \binom{n}{k}, A = \ksparsen.$
\begin{enumerate}
    \item If $d\longrightarrow +\infty,$ as $r= \sqrt{2d}(1 + o_d(1)),$ it follows that $\Delta\longrightarrow 1.$ We recover the fact $\maxmargin(d, \ksparsen)\ge \frac{1+o_k(1)}{2\sqrt{k}}.$ Note, however, that this construction requires $d = \Omega(k^2\log^2 n)$ for margin $\Theta(1/\sqrt{k})$ even though we know that dimension $\Theta(k\log n)$ is sufficient.
    \item If $d = \frac{k^2}{2}(1+\alpha)^2$ for some $\alpha>0,$ we get 
    $\maxmargin(\frac{k^2}{2}(1 + \alpha)^2, \ksparsen)\ge \frac{1}{2\sqrt{k}}n^{-\frac{2(k+1)}{\alpha k}}.$ Thus, even for growing $k,$ we obtain any inverse polynomial margin in dimension $\Theta(k^2).$
\end{enumerate}

\textbf{Part 4: Experiments.} We test what margins are obtained by low-dimensional representations trained via contrastive learners in the free embedding model of \cite{weller2025theoretical}. Concretely, we set $k = 2$ and, for varying $n$ and $d,$ train the embeddings $\{U_j\}_{j = 1}^N$ and $\{V_i\}_{i = 1}^n\in \mathbb{S}^{d-1}$ using either the InfoNCE loss \cite{oord2018representation}, as in \cite{weller2025theoretical}, or the sigmoid loss used in SigLIP \cite{zhai23siglip}, with positive (negative) samples determined by $(\ksparsen)_{ji} = 1$ ($(\ksparsen)_{ji} = 0$). 
See Appendix~\ref{sec:freeembedding} for full experimental details.

The work \cite{weller2025theoretical} uses the InfoNCE experiments to suggest that $d = \Theta(n^{1/3})$ is needed for any positive margin. Our experiment on InfoNCE in Figure~\ref{fig:min-nonzero-margin} is consistent with \cite[Figure 2]{weller2025theoretical}. 

What is surprising in light of \cite{weller2025theoretical} is that the sigmoid loss strongly outperforms InfoNCE. The sigmoid loss needs a much lower dimension, close to the theoretically optimal $5$ and nearly independent of $n,$ to obtain a positive margin (Figure~\ref{fig:min-nonzero-margin}). Likewise, the margins obtained by the sigmoid loss are consistently higher  (Figure~\ref{fig:max-margin-vs-logn}). The behavior for larger $k$ is similar, see Appendix~\ref{sec:freeembedding}.

A plausible explanation is that the global minimizers of the sigmoid loss (with trainable inverse temperature) coincide exactly with the embeddings of $A$ in the sense of Definition~\ref{def:maindefinition}. However,  this is not the case for InfoNCE. See Proposition~\ref{prop:marginsigmoid} and Appendix~\ref{sec:freeembedding}.

Of course, our experiments only consider margin in the free embedding model. It is possible that for properties of the embeddings relevant to tasks beyond retrieval or for certain real datasets, the InfoNCE loss has advantages over the sigmoid loss. 

%% file: JL.tex
\section{Low Dimensions Suffice: Proof of 
\texorpdfstring{Theorem~\ref{thm:main}}{Main Theorem}}
\label{sec:mainthmproof}
We need the following two classical results to prove Theorem~\ref{thm:main}.
\begin{theorem}[Johnson-Lindenstrauss Lemma \cite{johnson1984extensions}]
\label{thm:JL}
Suppose that $\mathcal{X}\subseteq \mathbb{R}^D$ is a set of $n$ points and $\epsilon>0.$ Then, there exists a matrix $\Pi\in \mathbb{R}^{ \frac{32\log n}{\epsilon^2}\times D}$ such that for every $x\in \mathcal{X},$ it holds that 
$(1-\epsilon)\|x\|^2_2\le \|\Pi x\|^2_2\le \|x\|^2_2.$
\end{theorem}
\begin{theorem}[Maurey's Lemma, \cite{pisier1981remarques}]
\label{thm:Murrayslemma}
Suppose that $\mathcal{Y} = \conv(\{Y_i\}_{i = 1}^n)$ for some points $\{Y_i\}_{i = 1}^n\in \mathbb{R}^D$ of norm at most 1. Let $T\in \mathbb{N}$ and let $x$ be any point inside $\mathcal{Y}.$ Then, there exist (not necessarily different) indices $i_1, i_2, \ldots, i_T$ such that 
$
\|x - \frac{1}{T}\sum_{j = 1}^T Y_{i_j}\|_2\le \frac{1}{\sqrt{T}}.
$
\end{theorem}
\begin{proof}[Proof of Theorem~\ref{thm:main}] Without loss of generality, suppose that $n\le N.$ Otherwise, we can argue about $A^T$ instead of $A.$
Suppose that there exist unit vectors $\{U_j\}_{j = 1}^N $ and $\{V_i\}_{i = 1}^n$ of some dimension $D$ such that $(-1)^{A_{ji}+1}\langle U_j, V_i\rangle \ge m$ for every $j \in [N], i \in [n].$ By Lemma~\ref{lem:convexhullequivalentofmargin}, for every $j \in [N]$ and 
$x\in \mathcal{X}_j\coloneqq \conv(\{(-1)^{A_{ji}+1}V_i\}),$ it holds that $\|x\|_2\ge m.$ In particular, for some $T$ that we will choose later, $\|x\|_2\ge m$ holds for every $x$ in the following set:
$$
\mathcal{Z}_T \coloneqq 
\Big\{
\frac{1}{T}\sum_{\ell = 1}^T(-1)^{A_{j,i_\ell}+1}V_{i_\ell} \; : \; 
j \in [N], i_1, i_2, \ldots, i_T \in [n]
\Big\}.
$$

Now, observe that $|\mathcal{Z}_T|\le (2n)^T.$ Indeed, this holds, since there are $n^T$ choices for $i_1, i_2, \ldots, i_T \in [n]$ and at most $2^T$ choices for the signs $(-1)^{A_{j,i_\ell}+1}.$

We will apply the dimensionality reduction Theorem~\ref{thm:JL} to $\mathcal{Z_T}.$ Let $T = \frac{128}{m^2\epsilon^2}$ so that $T\log n = \Theta(\epsilon^{-2}m^{-2}\log n),$
as desired.
By Theorem~\ref{thm:JL}, there exists some $\Pi \in \mathbb{R}^{O(\epsilon^{-4}m^{-2}\log n)\times D}$ such that for every $v\in \mathcal{Z}_T,$ it holds that 
$(1-\epsilon/2)^2\|v\|_2^2\le \|\Pi v\|^2_2\le \|v\|^2_2.$

Define $\widetilde{V_i}\coloneqq \Pi V_i.$ For every $i,$ at least one of $V_i$ and $-V_i$ belongs to $\mathcal{Z}_T$  (take $i = i_1 = \cdots = i_T$). Hence, $1-\epsilon/2 \le \|\widetilde{V}_i\|_2\le 1.$ To obtain the final $V$-vectors, we will simply normalize.

Before that, we demonstrate the existence of vectors $\widetilde{U}_j$ for $\{\widetilde{V}_i\}_{i = 1}^n$
via Lemma~\ref{lem:convexhullequivalentofmargin}. Concretely, take some $j \in [N]$ and consider $\widetilde{X_j} = \conv(\{(-1)^{A_{ji}+1}\widetilde{V_i}\}).$ Let $\widetilde{x}\in \widetilde{X_j}.$ It is enough to show that $\|\widetilde{x}\|_2\ge (1 - \epsilon)m.$ To this end, we use the JL Lemma on $\mathcal{Z}_T$ and Theorem~\ref{thm:Murrayslemma}.  Namely, by~\ref{thm:Murrayslemma}, there exist some indices $i_1, i_2, \ldots, i_T$ such that 
$$
\Big\|\widetilde{x} -\frac{1}{T}\sum_{\ell = 1}^T (-1)^{A_{j,i_\ell}+1}\widetilde{V_{i_\ell}}\Big\|_2\le 
T^{-1/2}\le 
m\epsilon/2.
$$
On the other hand, by the guarantees of Theorem~\ref{thm:JL}, 
$$
\Big\|\frac{1}{T}\sum_{\ell = 1}^T (-1)^{A_{j,i_\ell}+1}\widetilde{V_{i_\ell}}\Big\|_2 = 
\Big\|\Pi \Big(\frac{1}{T}\sum_{\ell = 1}^T (-1)^{A_{j,i_\ell}+1}V_{i_\ell}\Big)\Big\|_2\ge 
(1-\epsilon/2)
\Big\|\frac{1}{T}\sum_{\ell = 1}^T (-1)^{A_{j,i_\ell}+1}V_{i_\ell}\Big\|_2.
$$
Finally, as  $\frac{1}{T}\sum_{\ell = 1}^T (-1)^{A_{j,i_\ell}+1}V_{i_\ell}\in \mathcal{X}_j,$ we know that $\|\frac{1}{T}\sum_{\ell = 1}^T (-1)^{A_{j,i_\ell}+1}V_{i_\ell}\|_2\ge m.$
Hence,
$$
\|\widetilde{x}\|_2 \ge 
\|\frac{1}{T}\sum_{\ell = 1}^T (-1)^{A_{j,i_\ell}+1}\widetilde{V_{i_\ell}}\|_2 - 
\|\widetilde{x} -\frac{1}{T}\sum_{\ell = 1}^T (-1)^{A_{j,i_\ell}+1}\widetilde{V_{i_\ell}}\|_2\ge 
m(1-\epsilon/2) -
m\epsilon/2 = m(1-\epsilon). 
$$
Applying Lemma~\ref{lem:convexhullequivalentofmargin}, there exist unit vectors $\{\widetilde{U}_j\}_{j = 1}^N$ in dimension $O(\epsilon^{-4}m^{-2}\log n)$ such that 
$(-1)^{A_{ji}+1}\langle \widetilde{V_i}, \widetilde{U_j}\rangle\ge m(1-\epsilon)\quad \forall i,j.$ Clearly, this property remains true even if we replace $\{\widetilde{V_i}\}^n_{i = 1}$ by $\{\widetilde{V_i}/\|\widetilde{V_i}\|_2\}^n_{i = 1}$ as $\|\widetilde{V_i}\|_2\le 1 \quad \forall i.$\end{proof}

%% file: LowerBoundsv1.tex
\section{Impossibility Results: Proof of Theorem \ref{thm:mainlowerbound}}
\label{sec:Lower Bounds}
We need the following technical ingredients toward constructing large sphere packings. 
\begin{lemma}[\cite{foucart2010gelfand}]
\label{lem:largedisjointfamilies}
    Let $n, s \in \mathbb{N}$ with $s < n$. Then there exists a family $\mathcal{F}$ of subsets of $[n]$ with the following properties. (1) Every set in $\mathcal{F}$ consists of exactly $s$ elements. (2) For all $T, T' \in \mathcal{F}$ with $T \neq T'$, $|T\cap T'| < \frac{s}{2}.$ (3) $|\mathcal{F}| \geq  (\frac{n}{4s}  )^\frac{s}{2}.$
\end{lemma}

\begin{lemma}\label{lem:reductiontosparselowerbound}
    Suppose that $n \ge 3k.$ Let $\{V_i\}_{i =1}^n$, $\{U_j\}_{j = 1}^N$ be any dimension-$d$ margin-$m$ embedding of $\ksparsen.$ Then, for any $x \in \mathbb{R}^n$ with $\|x\|_0 \leq k$, it holds that 
    $\|Vx\|_2 \geq m\|x\|_1.$
\end{lemma}
\begin{proof} This follows immediately from Lemma~\ref{lem:convexhullequivalentofmargin}. By homogeneity, it is enough to prove the fact when $\|x\|_1 = 1.$ Now, take some $j$ such that $(-1)^{A_{ji}+1}$ and $x_i$ have the same sign whenever $x_i\neq 0.$ Then, $Vx\in \conv(\{(-1)^{A_{ji}+1}V_i\}_{i = 1}^n)$ and the result follows.
\end{proof}
\begin{proof}[Proof of Theorem \ref{thm:mainlowerbound}]
    Let $V$ be as in Lemma \ref{lem:reductiontosparselowerbound}. We first show that $m \sqrt{k} \leq 1.$ We already know even the sharper inequality $m\le \frac{1+o_k(1)}{2\sqrt{k}}$ 
    from Corollary~\ref{cor:sqrtkmargin} but will nevertheless give a proof as the proof technique will be useful later on.
    Fix any set $T \subset [n]$ with $|T|=k$, and let $\epsilon$ be a uniformly random vector in  $\{\pm 1\}^T.$ Define $x = \sum_{i \in T}\epsilon_i e_i$. By Lemma \ref{lem:reductiontosparselowerbound}, we have
$\|Vx\|_2 = \|\sum_{i \in T}\epsilon_i V_i\|_2 \geq m\|x\|_1 = mk.$
On the other hand, 
$$
\expect
\Big[\|\sum_{i \in T}\epsilon_i V_i\|_2^2
\Big] = 
\expect
\Big[
\sum_{i \in T}\|V_i\|_2^2 + \sum_{\substack{i,\ell \in T\\ i\neq \ell}}\epsilon_i\epsilon_\ell \langle V_i,V_\ell\rangle
\Big]
= \sum_{i \in T}\|V_i\|_2^2  = k,$$
where we used that $\epsilon_i$ and $\epsilon_\ell$ are independently chosen from $\pm 1$ for $i \neq \ell$. Thus, there exists an $\epsilon \in \{\pm1\}^T$ such that $\|\sum_{i \in T}\epsilon_iV_i\|_2^2\le k$. Combining with $\|\sum_{i \in T}\epsilon_iV_i\|_2\ge mk$, we get $m^2k^2\le k$, and therefore $m\sqrt{k}\le 1$. Assume $k \geq 8$ as otherwise the theorem holds after decreasing $C$.

Now, we will construct the large sphere packing. Set $s\coloneqq \lfloor \frac{k}{8}\rfloor$. By Lemma \ref{lem:largedisjointfamilies} there exists a set $\mathcal{F} \subset \binom{[n]}{s}$ such that for any $T \neq T'$ with $T, T' \in \mathcal{F}$, $|T\cap T'| < \frac{s}{2}$ and $|\mathcal{F}| \geq  (\frac{n}{4s}  )^\frac{s}{2}$. 

For a $T \in \mathcal{F}$, define the sign vector $\sigma^T \in \{\pm 1\}^T$ such that $ \|\sum_{i \in T} \sigma_i^T V_i  \|_2$ is minimized. Denote $y_T \coloneqq \sum_{i \in T} \sigma_i^T V_i.$
By the same argument as above, $\|y_T\|_2^2 = \|\sum_{i \in T} \sigma_i^T V_i\|^2_2\le s.$  In other words, $y_T \in B_d(0,\sqrt{s})$ for all $T \in \mathcal{F}$ where $B_d(0,r)$ is the $L_2$ ball of radius $r$ centered at 0.

Define $x_T\coloneqq\sum_{i \in T}\sigma_i^T e_i$ so that $Vx_T = y_T$. Consider any $T, T' \in \mathcal{F}$ with $T \neq T'$. Since $|T \cap T'| < \frac{s}{2}$ and $s \leq \frac{k}{8}$, it follows that $\|x_{T} - x_{T'}\|_0 \leq k$. Hence, by Lemma \ref{lem:reductiontosparselowerbound} we get 
$$\|V(x_T - x_{T'})\|_2 = \|Vx_T - Vx_{T'}\|_2 \geq m\|x_T - x_{T'}\|_1.$$
Moreover, as every coordinate of $x_T, x_{T'}$ is in $\{0, \pm 1\},$
$$\|x_T-x_{T'}\|_1 \ge |T\triangle T'| = |T|+|T'|-2|T\cap T'| > 2s-s = s.$$
Combining the last two inequalities, we obtain 
$$
\|y_T - y_{T'}\|_2 = 
\|V(x_T - x_{T'})\|_2 = \|Vx_T - Vx_{T'}\|_2 \geq m\|x_T - x_{T'}\|_1 \geq ms.$$

Thus, the vectors $\{y_T\}_{T \in \mathcal{F}}$ are pairwise $ms$ separated in the Euclidean ball $B_d(0, \sqrt{s})$. So, the vectors $\tilde{y}_T\coloneqq \frac{y_T}{\sqrt{s}}$ are pairwise $m\sqrt{s}$ separated in the ball $B_d(0,1)$. By a standard volume bound (for example  \cite[Corollary 4.2.11]{Vershynin2026}), it follows that
$|\mathcal{F}| \leq  (1 + \frac{2}{m\sqrt{s}}  )^d.$ Using $|\mathcal{F}| \geq  (\frac{n}{4s}  )^\frac{s}{2},$ 
$$ (\frac{n}{4s}  )^\frac{s}{2} \leq |\mathcal{F}| \leq  (1 + \frac{2}{m\sqrt{s}}  )^d \implies d \log  (1 + \frac{2}{m\sqrt{s}}  ) \geq \frac{s}{2}\log{\frac{n}{4s}}.$$

By compactness, there exists an embedding attaining margin $\maxmargin(d,\ksparsen)$. Taking $m=\maxmargin(d,\ksparsen)$ in the preceding argument and using $s=\Theta(k)$, the conclusion follows.
\end{proof}

%% file: Construction.tex
\section{Construction: Proof of \texorpdfstring{Theorem~\ref{thm:main_construction}}{Main Construction Theorem}}
\label{sec:proofofconstruction}
The construction will be based on lifts of random spherical vectors. To this end, we will need the following standard facts about the distribution of random spherical vectors.

\begin{fact}[\cite{frankl1990some}]
\label{fact:betadistandpsherical}
Let $z= (z_1, \ldots, z_r)$ be a uniformly random vector on $\mathbb{S}^{r-1}.$ Then, for any $1\le s \le r,$ the distribution of $z_1^2 + \ldots + z_s^2$ is the distribution $\Betadist(s/2, (r-s)/2)$ with density 
$$
f_{s/2, (r-s)/2}(x) = \frac{\Gamma(\frac{r}{2})}{\Gamma(\frac{s}{2})\Gamma(\frac{r-s}{2})}x^{\frac{s}{2}-1}(1-x)^{\frac{r-s}{2}-1}.
$$
\end{fact}
An immediate corollary is the following tail estimate which we prove in Appendix~\ref{appendix:MissingDetalsProofOfConstruction}.
\begin{proposition}
\label{prop:betaconcentration}
Let $X\sim \Betadist(s/2, (r-s)/2)$ for some $2\le s \le r-2$ and $\delta \in(0,1].$ Then,
$$
\prob[X\le \delta]\le 
\frac{\Gamma(\frac{r}{2})}{\Gamma(\frac{s}{2})\Gamma(\frac{r-s}{2})}\delta^{s/2}.
$$
\end{proposition}

\begin{proof}[Proof of Theorem~\ref{thm:main_construction}] Let $r = \lfloor \frac{\sqrt{8d-7}-1}{2}\rfloor$ so that $d\ge \frac{r(r+1)}{2} + 1$ and $r\ge k+2.$ Define the map $\vectorize:\mathbb{R}^{r\times r}\longrightarrow \mathbb{R}^{r(r+1)/2}$ as the vector $\vectorize(M)$ with entries $(M_{ii})_{i = 1}^r$ and $(\sqrt{2}M_{ij})_{i<j}.$ We use the fact that for two symmetric matrices $M_1, M_2,$ it holds that $\langle \vectorize(M_1), \vectorize(M_2)\rangle = \langle M_1, M_2\rangle_F.$

Now, we can get to the actual construction. Let $a_1, a_2, \ldots, a_n$ be i.i.d. vectors uniformly drawn from $\mathbb{S}^{r-1}.$ Let $\tV_i = \vectorize(a_ia_i^T).$ Let $P_j\in \mathbb{R}^{r\times r}$ be the symmetric matrix projecting onto $\mathsf{span}(\{a_i\; : \; A_{ji} = 1\}).$ Let $\tU_j = \vectorize(P_j).$ Note that for every $i,j,$ $\|\tV_i\|_2 = \|a_ia_i^T\|_F = 1$ and 
$\|\tU_j\|_2 = \|P_j\|_F = \sqrt{k}$ as it is a projection matrix on a subspace of dimension $k.$ The actual $U_j,V_i$ that we will construct are simple linear transformations of $\widetilde{U_j}, \widetilde{V_i}.$ First, however, we analyze the inner products between the constructed $\tU_j$ and $\tV_i.$ 

If $A_{ji} = 1,$ then $\langle \tV_i, \tU_j\rangle = \langle P_j, a_ia_i^T\rangle = \langle P_ja_i, a_i\rangle  = 
\langle a_i, a_i\rangle = 1$ as $a_i \in \mathsf{span}(\{a_i\; : \; A_{ji} = 1\}).$ 

If $A_{ji} = 0,$
$\langle \tV_i, \tU_j\rangle = \langle P_j, a_ia_i^T\rangle = \langle P_ja_i, a_i\rangle  = 
\langle P_ja_i, P_ja_i\rangle = \|P_ja_i\|_2^2 \sim \Betadist(k/2, (r-k)/2)$ by Fact~\ref{fact:betadistandpsherical} since $P_j$ is a projection to a $k$-dimensional subspace of $\mathbb{R}^{r}$ independent of $a_i.$ Therefore, $1 - \|P_ja_i\|_2^2 \sim \Betadist((r-k)/2, k/2).$ 

Now, let $\Delta = \Bigg(\frac{\Gamma(\frac{k}{2})\Gamma(\frac{r-k}{2})}{2N(n-k)\Gamma(\frac{r}{2})}\Bigg)^{\frac{2}{r-k}}.$ By Proposition~\ref{prop:betaconcentration},
$$
\prob[\langle \tV_i, \tU_j\rangle\ge 1-\Delta ] = 
\prob[1 - \|P_ja_i\|_2^2 \le \Delta]\le 
\frac{\Gamma(\frac{r}{2})}{\Gamma(\frac{k}{2})\Gamma(\frac{r-k}{2})}\Delta^{\frac{r-k}{2}}= \frac{1}{2N(n-k)}.
$$
By union bound, it follows that there exists a realization of $a_1, a_2, \ldots, a_n$ such that 
$\langle \tV_i, \tU_j\rangle \le 1 - \Delta$ for every $j,i$ such that $A_{ji} = 0.$

The final construction in dimension $\frac{r(r+1)}{2} + 1\le d$ is formed by setting $\beta = \sqrt{\frac{2\sqrt{k}}{2(\sqrt{k}+1) - \Delta}}$ and 
\begin{align*}
 V_i = (\beta\tV_i, \sqrt{1-\beta^2}),
 U_j = (\frac{\beta}{\sqrt{k}}\tU_j, -\sqrt{1-\beta^2}).
\end{align*}
As $\|\tV_i\|_2 = 1, \|\tU_j\|_2 = \sqrt{k},$ it follows that $\|U_j\|_2 = 1, \|V_i\| = 1.$ Using $\langle \tV_i, \tU_j\rangle = 1,$ we can check that whenever $A_{ji} = 1,$
$\langle U_j, V_i\rangle = \frac{\Delta}{2(\sqrt{k}+1) - \Delta}.$ Whenever $A_{ji} = 0,$ using that 
$\langle \tV_i, \tU_j\rangle \le 1 - \Delta,$ we conclude that 
$\langle U_j, V_i\rangle \le - \frac{\Delta}{2(\sqrt{k}+1) - \Delta}.$
All that is left to show is the asymptotics of $\Delta.$ This is a standard corollary of Stirling's approximation which we prove for completeness in Appendix~\ref{appendix:MissingDetalsProofOfConstruction}.
\end{proof}

%% file: Limitations.tex
\section{Limitations, Broader Impact, and LLM Usage Statement}
\label{sec:limitations}
One gap in our work is that in the regime $d= o(k\log(n/k)),$ the lower bound on $\maxmargin(d, \ksparsen)$ in Theorem~\ref{thm:main_construction} and upper bound in Theorem~\ref{thm:mainlowerbound} do not match. A second limitation comes from the fact that our current experiments are in the free embedding model rather than on practical datasets and architectures. We believe that both of these are exciting directions for future study. 

We are not aware of any direct negative societal impact of the work. 

We used LLMs for completing some proofs, finding related works, generating code, and editing. We wrote the current manuscript without any LLM usage and only then used
LLMs for editing. We have verified all code produced by LLMs. More detail about LLM usage is given in Appendix~\ref{sec:llmusage}.  

%% file: SimpleAppendix.tex
\section{Proofs of Preliminary Claims}
\label{appendix:somesimpleclaims}
\subsection{Relative Bias and Margin}

\begin{proposition}[On Relative Bias and Margin]
\label{prop:rb}
Suppose that $A\in \{0,1\}^{N\times n}$ and there exists a margin-$m$ relative bias-$\tau$ embedding of $A$ in dimension $d.$ Then, there exists a relative bias-$0$ margin- $\frac{m}{{1+|\tau|}}\ge\frac{m}{{2}}$ embedding of $A$ in dimension $d+1.$ 
\end{proposition}
\begin{proof}
Suppose that $\{U_j\}_{j = 1}^N, \{V_i\}_{i = 1}^n$ is a margin-$m,$ relative bias-$\tau$ embedding of $A.$ Then, $\langle U_j, V_i\rangle \ge \tau+m$ whenever $A_{ji} = 1$ and  
$\langle U_j, V_i\rangle \le \tau-m$ otherwise. 

Now, let $\widetilde{V_i}\coloneqq (\frac{1}{\sqrt{1+|\tau|}}V_i, \frac{\sqrt{|\tau|}}{\sqrt{1 + |\tau|}})$ and $\widetilde{U_j}\coloneqq (\frac{1}{\sqrt{1+|\tau|}}U_j, -\textsf{sign}(\tau)\frac{\sqrt{|\tau|}}{\sqrt{1 + |\tau|}}).$ One can check that $\{\widetilde{U_j}\}_{j = 1}^N, \{\widetilde{V_i}\}_{i = 1}^n$ are unit norm and $\langle \widetilde{U_j}, \widetilde{V_i}\rangle \ge \frac{m}{{1+|\tau|}}$ whenever $A_{ji} = 1$ and  
$\langle \widetilde{U_j}, \widetilde{V_i}\rangle \le -\frac{m}{{1+|\tau|}}$ otherwise. Finally, note that $\tau \le \tau+m \le \langle U_j, V_i\rangle\le \|U_j\|_2\times \|V_i\|_2 = 1$ whenever $A_{ji} = 1.$ Similarly, $\tau\ge -1.$
\end{proof}

\subsection{Proof of Lemma~\texorpdfstring{\ref{lem:convexhullequivalentofmargin}}{Convex Hull Lemma}}
$(\Longrightarrow)$ Take any $x\in \conv(\{\sigma_i V_i\}_{i = 1}^n).$ By definition, for some non-negative $\{\alpha_i\}_{i = 1}^n$ with sum 1, $x = \sum_{i= 1}^n\alpha_i \sigma_iV_i.$ As $U$ is unit and $\sigma_i \langle U, V_i\rangle\ge m \quad \forall i,$ it follows that 
$$\|x\|_2\ge \langle U, x\rangle  =
\langle U, \sum_{i= 1}^n\alpha_i \sigma_iV_i\rangle
= \sum_{i= 1}^n\alpha_i\sigma_i
\langle U, V_i\rangle\ge 
\sum_{i= 1}^n\alpha_i m = m.
$$
$(\Longleftarrow)$ Take $U_1$ to be the vector of minimal norm inside $\conv(\{\sigma_i V_i\}_{i = 1}^n).$ We know that $\|U_1\|_2= m.$ Let $U'\coloneqq \frac{U_1}{\|U_1\|_2}.$ We claim that the unit vector $U'$ satisfies the desired property. For the sake of contradiction, suppose that this is not true. 
Then, there exists some $i$ such that $\langle \sigma_i V_i, U'\rangle <m.$ Equivalently, 
$\langle \sigma_i V_i, U_1\rangle <m^2 = \|U_1\|^2_2.$ Thus, 
$\langle \sigma_i V_i- U_1, U_1\rangle<0.$ This, however, means that for some sufficiently small $\alpha>0,$ it holds that the vector $\alpha \sigma_i V_i + (1-\alpha)U_1,$ which is inside $\conv(\{\sigma_i V_i\}_{i = 1}^n),$ has smaller norm than $\|U_1\|_2 = m,$ which is a contradiction. Indeed, note that 
\begin{align*}
\frac{d}{d\alpha}\Bigg|_{\alpha = 0} \|\alpha \sigma_i V_i + (1-\alpha)U_1\|^2_2 = 2\langle \sigma_i V_i - U_1, U_1\rangle <0.\qedhere
\end{align*}



%% file: AppendixMissingConstruction.tex
\section{Missing Details from \texorpdfstring{Section~\ref{sec:proofofconstruction}}{Proof of Construction}}
\label{appendix:MissingDetalsProofOfConstruction}
\subsection{Proof of Proposition~\texorpdfstring{\ref{prop:betaconcentration}}{Beta Concentration}}
   \begin{align*}
       & \prob[X\le \delta] = \int_0^\delta f_{s/2, (r-s)/2}(x)dx = 
   \frac{\Gamma(\frac{r}{2})}{\Gamma(\frac{s}{2})\Gamma(\frac{r-s}{2})}\int_0^\delta x^{\frac{s}{2}-1}(1-x)^{\frac{r-s}{2}-1} dx\\
   &\le 
   \frac{\Gamma(\frac{r}{2})}{\Gamma(\frac{s}{2})\Gamma(\frac{r-s}{2})}\int_0^\delta x^{\frac{s}{2}-1}dx\le 
   \frac{\Gamma(\frac{r}{2})}{\Gamma(\frac{s}{2})\Gamma(\frac{r-s}{2})}\int_0^\delta \delta^{\frac{s}{2}-1}dx\le 
   \frac{\Gamma(\frac{r}{2})}{\Gamma(\frac{s}{2})\Gamma(\frac{r-s}{2})}
   \delta^{s/2}.\qedhere
   \end{align*}
\subsection{Gamma Function Asymptotics}
To finish the asymptotics of $\Delta$ from Section~\ref{sec:proofofconstruction}, all we need to show is that for real numbers $x,y>1/2,$ with $
x=\frac{k}{2}, y=\frac{r-k}{2},$ the following binomial-type approximation holds. 

\begin{lemma}
Suppose that $x,y\ge 1/2$ are real numbers. Then,
$$
\Bigg(\frac{\Gamma(x)\Gamma(y)}{\Gamma(x+y)}\Bigg)^{1/y} = \Theta\Big(\frac{y}{x+y}\Big).
$$
\end{lemma}
\begin{proof}
By Stirling's formula, uniformly for \(z\ge \tfrac12\),
$
\Gamma(z)=\Theta\big( z^{\,z-\frac12}e^{-z}\big).
$
Hence
\[
\left(\frac{\Gamma(x)\Gamma(y)}{\Gamma(x+y)}\right)^{1/y}
=\Theta
\left(\frac{x^{x-\frac12}y^{y-\frac12}}{(x+y)^{x+y-\frac12}}\right)^{1/y}
=
\frac{y}{x+y}
\left(1+\frac{y}{x}\right)^{-x/y}
\left(\frac{x+y}{xy}\right)^{1/(2y)}.
\]

Denote \(t=x/y>0\). Then,
\[
\left(1+\frac{y}{x}\right)^{-x/y}=(1+t^{-1})^{-t}=\Theta(1),
\]
since \(1\le (1+t^{-1})^t\le e\). Also, because \(x,y\ge \tfrac12\),
\[
\frac1{2y}\log\!\left(\frac{x+y}{xy}\right)
=
\frac1{2y}\log\!\left(\frac1x+\frac1y\right)
\]
is bounded above by \(\log 4\), and below by
\[
\frac1{2y}\log\!\left(\frac1y\right)
=
-\frac{\log y}{2y}\ge -\frac1{2e},
\]
so
\[
\left(\frac{x+y}{xy}\right)^{1/(2y)}=\Theta(1).
\]
Therefore
\[
\left(\frac{\Gamma(x)\Gamma(y)}{\Gamma(x+y)}\right)^{1/y} = 
\Theta\Big( \frac{y}{x+y}\Big),\]
as desired.
\end{proof}

%% file: MarginAlgebraic.tex
\section{Margin of the Construction of \texorpdfstring{\cite{AlonFranklRodl1985}}{Alon,Frankl,Rold}}
\label{sec:marginalgebraic}
\subsection{Proof of Theorem~\texorpdfstring{\ref{thm:signrankforsparse}}{Sign Rank Theorem}}
\label{appendix:signrank}
\begin{proof}[Proof of Lower Bound in Theorem~\ref{thm:signrankforsparse}, following \cite{alon2016sign}] The goal is to show that $\signrank(\ksparsen)\ge 2k+1.$ Towards this, we use the notion of dual sign rank. 

Consider any matrix $A\in \{0,1\}^{N\times n}.$ The set of columns indexed by $C$ is antipodally shattered if for every $v\in \{0,1\}^{|C|},$ either $v$ or $\neg v $ (in which all coordinates are flipped) appears as a row in $A$ restricted to $C.$ Now, according to \cite[Corollary 1.6]{alon2016sign}, the dual sign rank $\dsignrank(A)$ coincides with the largest possible $c$ for which a set of antipodally shattered columns $C$ with size $c$ exists. 

Clearly, for $\ksparsen$ then, $\dsignrank(\ksparsen) = 2k+1$ since among the first $2k+1$ columns, every vector of Hamming weight at most $k$ exists. Finally, the definition of dual sign rank immediately implies that for any $A,$ $\dsignrank(A)\le \signrank(A).$ It follows that $\signrank(\ksparsen)\ge \dsignrank(\ksparsen)\ge 2k+1.$ 
\end{proof}

\begin{proof}[Proof of Upper Bound in Theorem~\ref{thm:signrankforsparse}, following \cite{AlonFranklRodl1985}]
The construction is based on a Vandermonde embedding of the document embeddings. We will not assume for now that $\{U_j\}_{j = 1}^{\binom{n}{k}}$ and 
$\{V_i\}_{i = 1}^n$ are unit norm since rescaling the norms will not change the signs.

Take any real numbers $t_1 < t_2<\cdots <t_n.$ Let 
$V_i = (1 ,t_i, t_i^2 \cdots ,t_i^{2k}).$ 

Now, take any $j\in [\binom{n}{k}].$
We want to find a $U_j = (u_{j,0}, u_{j,1}, \ldots, u_{j, 2k})\in \mathbb{R}^{2k+1}$ such that 
$\langle U_j, V_{i}\rangle > 0 $ 
whenever $(\ksparsen)_{ji} = 1$ and $\langle U_j, V_{i}\rangle < 0 $ otherwise. Toward this end, observe that 
$$
\langle U_j, V_{i}\rangle = 
\sum_{\ell = 0}^{2k}
u_{j,\ell}t_i^\ell = P_j(t_i)
$$
for the polynomial of degree $2k$ with coefficient vector 
$(u_{j,0}, u_{j,1}, \ldots, u_{j, 2k}).$

Hence, it is enough to find a polynomial of degree $2k$ $P_j$ such that 
$P_j(t_i)>0$ whenever $(\ksparsen)_{ji} = 1$ and 
$P_j(t_i)<0$ otherwise. 

Note, however, that the polynomial 
$P_j$ only needs to take $k$ positive values on the set 
$\{t_1, t_2, \ldots, t_n\}$ as $\ksparsen$ has $k$-sparse rows. Hence, the polynomial needs to change its sign on at most $2k$ of the intervals $[t_1, t_2], [t_2, t_3], \ldots, [t_{n-1}, t_n].$ As for any $2k$ real numbers, there exists a polynomial of degree $2k$ with roots those numbers, we are done. 
\end{proof}

\subsection{Proof of Corollary~\texorpdfstring{\ref{cor:marginalgebraic}}{AlgebraicMargin}}
We make a concrete choice of $t_1, t_2, \ldots, t_n$ and the polynomials from Appendix~\ref{appendix:signrank}.

\begin{proof}[Proof of Corollary~\ref{cor:marginalgebraic}] In the construction above, choose \(
 t_i := -1 + \frac{2(i-1)}{n-1}
\) for each $i\in [n]$ so that $\tilde{V}_i = (1 ,t_i, t_i^2 \cdots ,t_i^{2k})$ and 
$V_i = \tilde{V}_i/\|\tilde{V}_i\|_2.$ Note that as $t_i \in [-1,1]$ for all $i,$ it holds that $1\le \|\tilde{V}_i\|_2\leq \sqrt{2k+1}$ for all $i.$

Now, take some $j\in [\binom{n}{k}].$ Our goal is to show that there exists a unit $U_j$ such that $\langle U_j, V_i\rangle\ge (2k)^{-1/4}(2n)^{-2k}.$

Let \(
1 \le i_1 < i_2 < \cdots < i_r \le n-1
\)
be the indices at which the sign changes, so that $A_{ji_\ell} \neq A_{ji_{\ell+1}}$ for
$\ell = 1, \dots, r$. Again, $r\le 2k.$ 

For each such index, set
\(
 s_\ell := \frac{t_{i_\ell} + t_{i_\ell+1}}{2}.
\)
Now define
\[
 p_j(s) := (-1)^{A_{ji_1}+1}\prod_{\ell=1}^r (s_\ell - s).
\]
Then $p$ has degree $r \le 2k$. Moreover, since the roots $s_1, \dots, s_r$ lie exactly in the gaps
where the sign pattern changes, the sign of $p_j(t_i)$ flips exactly when $A_{ji}$ flips. Since also
$p_j(t_1)$ has sign $(-1)^{A_{j1}+1}$, it follows that
$p_j(t_i)$ has sign $(-1)^{A_{ji}+1}$ for every $i\in [n].$ All we have to do is analyze the margin.

Denote 
\(
\Delta := \frac{2}{n-1}.
\)
Because the points $t_1, \dots, t_n$ are equally spaced and each $s_\ell$ is the midpoint of two
consecutive points, we have
\(
|t_i - s_\ell| \ge \frac{\Delta}{2} \ge \frac{1}{n}\) for all $i\in [n], \ell \in [r].$
Therefore, as $r\le 2k,$
\[
|p_j(t_i)| = \prod_{\ell=1}^r |t_i - s_\ell|
\ge \left(\frac{1}{n}\right)^r
\ge \left(\frac{1}{n}\right)^{2k}
\qquad \text{for all } i \in [n],
\]

Now, if the coefficient vector corresponding to $p_j$ is $(u_{j,0}, u_{j,1}, \ldots, u_{j,2k}),$ we will set $\tilde{U_j} = (u_{j,0}, u_{j,1}, \ldots, u_{j,2k})$ and 
$U_j = \tilde{U_j}/\|\tilde{U_j}\|_2.$ As the final margin will depend on $\|\tilde{U_j}\|_2$ we proceed to bound the norm.

Because each $s_\ell$ lies in $[-1,1]$, every coefficient of $p$ is a sum of products of numbers of
absolute value at most $1$. Hence the coefficient of $s^\ell$ has absolute value at most $\binom{r}{\ell}\le \binom{2k}{\ell}$,
and so
\[
\|\tilde{U_j}\|_2^2 = 
\sum_{\ell = 0}^{2k} u_{j,\ell}^2\le 
\sum_{\ell=0}^{2k} \binom{2k}{\ell}^2
= \binom{4k}{2k}\le \frac{2^{4k}}{\sqrt{3k}}.
\]

Altogether, this implies that for every $i,j,$ we have that 
\begin{align*}
& (-1)^{A_{ji} + 1}\langle U_j, V_i\rangle = 
(-1)^{A_{ji} + 1}
\frac{\langle \tilde{U_j}, \tilde{V_i}\rangle}{\|\tilde{U_j}\|_2 \times \|\tilde{V_i}\|_2}
= 
\frac{(-1)^{A_{ji} + 1}p_j(t_i)}{\|\tilde{U_j}\|_2 \times \|\tilde{V_i}\|_2} 
\\
&\ge 
\frac{n^{-2k}}{\sqrt{2k+1}\sqrt{2^{4k}/\sqrt{3k}}} = 
(2k)^{-1/4}2^{-2k}n^{-2k}.\qedhere
\end{align*}

\end{proof}

%% file: SpectralAndSnk.tex
\section{Spectral Bound on Margin Complexity}
\label{sec:spectral}
\subsection{Proof of Theorem~\texorpdfstring{\ref{thm:spectralboundonmargincomplexity}}{Spectral Bound}}
\begin{proof}[Proof of Theorem~\ref{thm:spectralboundonmargincomplexity}]
Suppose that $\{U_j\}_{j = 1}^N, \{V_i\}_{i = 1}^n$ is a margin-$m$ relative-bias-$0$ embedding of $A.$ Take any matrix $J\in \mathbb{R}^{N\times n}$ such that $J_{ji} \ge 0$ for all $i,j$ such that $A_{ji} = 1,$  
$J_{ji} \le 0$ for all $i,j$ such that $A_{ji} = 0,$ and 
$\sum_{j,i}|J_{ji}| = 1.$

Now, let $\alpha$ be any non-negative number. We have that 
    \begin{equation}
        \begin{split}
            0 & \le  \|U - \alpha JV\|_F^2 \\
              &  = \|U\|_F^2 - 2\alpha \langle U, JV\rangle + \langle JV, JV\rangle\\
              & = \|U\|_F^2 - 2\alpha \langle UV^T, J\rangle + \alpha^2\langle VV^T, J^TJ\rangle
              \Longrightarrow\\
            2\alpha \langle UV^T, J\rangle & \le  \|U\|_F^2 + \alpha^2\langle VV^T, J^TJ\rangle
        \end{split}
    \end{equation}
    Now, observe that
    $\|U\|_F^2  = N$ since $U$ is composed of $N$ unit-norm embeddings.
    Furthermore, as $VV^T$ is PSD,
    $$
    \langle VV^T, J^TJ\rangle \le \langle VV^T, \|J^TJ\|_{op} I\rangle = 
    \|J^TJ\|_{op} \times \|V\|_2^2 = \|J\|_{op}^2 \times n.
    $$
    Finally,
    $(UV^T)_{ji}J_{ji}\ge m |J_{ji}|$ for any $ji.$ Indeed, this is the case since if $A_{ji} = 1,$ both values are positive and $(UV^T)_{ji} \ge m.$ Otherwise, both values are negative and $(UV^T)_{ji} \le - m.$
    Altogether, we obtain 
    \begin{equation}
        \begin{split}
            & 2\alpha m \sum_{ji}|J_{ji}| \le  N  + \alpha^2 \|J\|_{op}^2 \times n \Longleftrightarrow\\
            & 2\alpha m \le N +  \alpha^2 \|J\|_{op}^2 \times n.
        \end{split}
    \end{equation}
    We used that $\sum_{ji}|J_{ji}|  = 1.$
Taking $\alpha = \sqrt{\frac{N}{n \|J\|^2_{op}}},$ we arrive at the desired conclusion.
\end{proof}

\subsection{Applications of Theorem~\texorpdfstring{\ref{thm:spectralboundonmargincomplexity}}{Spectral Bound}}
We now use Theorem~\ref{thm:spectralboundonmargincomplexity} to show Corollary~\ref{cor:sqrtkmargin}.

\begin{proof}[Proof of Corollary~\ref{cor:sqrtkmargin}]
First, we prove that $\maxmargin(+\infty, \ksparsen)\le \frac{1+o_k(1)}{2\sqrt{k}}.$ We will do so by choosing an appropriate matrix $J$ in Theorem~\ref{thm:spectralboundonmargincomplexity}. Concretely, let $y = \frac{1}{2}-\frac{1}{2}\frac{n-2k}{n-2k + 2\sqrt{(n-1)k(n-k)}}$ (so $y \approx \frac{1}{2}-\frac{1}{4\sqrt{k}}$), $N = \binom{n}{k},$ and set 
$$
J_{ji} = 
\begin{cases}
    & \frac{y}{Nk} \text{ if }(\ksparsen)_{ji} = 1,\\ 
    & -\frac{1-y}{N(n-k)} \text{ if }(\ksparsen)_{ji} = 0.\\ 
\end{cases}
$$
One can check that $\sum_{ji}|J_{ji}| = 1$ and $J^TJ = I \frac{1}{4Nnk}(1  + o_k(1)).$ Hence, $\|J\|_{op} = \frac{1}{2\sqrt{Nnk}}(1 + o_k(1))$ and we obtain the bound $m \le \frac{1 + o(1)}{2\sqrt{k}}.$

Now, we prove that $\maxmargin(+\infty, \ksparsen)\ge \frac{1+o_k(1)}{2\sqrt{k}}$ by exhibiting a construction in dimension $n+1.$
    Take $V_i = (\sqrt{1 - (4k)^{-1/2}}e_i,  (4k)^{-1/4})\in \mathbb{R}^{n+1}$ and $U_j = (\sqrt{1 - (4k)^{-1/2}}\sum_{i\; : \; A_{ji} = 1}\frac{1}{\sqrt{k}}e_i ,-(4k)^{-1/4})$ where $e_i$ form the standard basis. One can easily check that $\langle U_j, V_i\rangle \ge \frac{1+o_k(1)}{2\sqrt{k}}$ whenever $A_{ji} = 1$ and $\langle U_j, V_i\rangle \le - \frac{1+o_k(1)}{2\sqrt{k}}$ otherwise. Together with the previous result, this settles the margin complexity of $\ksparsen$ at $\frac{1+o_k(1)}{2\sqrt{k}}.$
\end{proof}

We can also recover the exact bound for $k = 1$ from \cite{bangachev2025globalminimizerssigmoidcontrastive} using Theorem~\ref{thm:spectralboundonmargincomplexity}.

\begin{corollary}
$\maxmargin(+\infty, I_n)\le \frac{n}{3n-4}.$
\end{corollary}
\begin{proof} Set $J = \frac{1}{3n-4}I - \frac{2}{n(3n-4)}11^T.$ Then, $\|J\|_{op} = \frac{1}{3n-4}$ from which the bound follows. 
\end{proof}

%% file: LargeMarginFromSparseCoding.tex
\section{Maximal Margin 
of \texorpdfstring{$\ksparsen$}{Snk} in 
dimension 
\texorpdfstring{$O(k\log n)$}{klogn} from Compressed Sensing}
\label{sec:largemarginviasparsecoding}
We begin by defining the restricted isometry property. 

\begin{definition}[\cite{CandesTao2005}] Let $M \in \mathbb{R}^{d\times n}$ be a matrix. For $1\le k \le n$ and $\delta_k\in (0,1)$ we say that $M$ satisfies $(\delta_k,k)$ restricted isometry property (RIP) if for any vector $v$ that is at most $k$-sparse, 
$$
\|v\|_2^2(1-\delta_k)\le 
\|Mv\|_2^2\le 
\|v\|_2^2(1+\delta_k).
$$    
\end{definition}
\begin{remark}
If $M$ satisfies $(\delta_k, k)$-RIP, it also satisfies $(\delta_m, m)$-RIP for any $m \le k, \delta_m \ge \delta_k.$\end{remark}

We now state a theorem of \cite{CandesTao2005} in the language of our work.   
\begin{theorem}[Modification of \cite{CandesTao2005}]
\label{thm:candestaocertificates}
Suppose that the matrix $V = (V_1 V_2 \cdots V_n)$ satisfies
$(\delta_{3k},3k)$-RIP such that
$\delta_{3k}\le 1/3.$
Let $c$ be a vector with support $T$ of size at most $k.$ Then, there exists a real vector $U_j$ such that:
\begin{enumerate}
    \item $\langle U_j, V_i\rangle = c_j$ whenever $i \in T.$
    \item $|\langle U_j, V_i\rangle| \le \frac{\delta_{3k}}{(1-2\delta_{3k})\sqrt{|T|}}\|c\|_2$ whenever $i \not \in T.$
    \item $\|U_j\|_2\le C(\delta_{3k})$ for some constant $C$ that depends only on $\delta_{3k}.$
\end{enumerate}
\end{theorem}
\begin{proof} 
Parts 1 and 2 appear directly in \cite[Lemma 2.2]{CandesTao2005} except that we have replaced orthogonality $\theta$ values by RIP $\delta$ values using \cite[Lemma 1.2]{CandesTao2005} and used the fact that $V$ also satisfies $(\delta_{3k},k)$-RIP and $(\delta_{3k},2k)$-RIP as in the above remark. 

Part 3 follows from the following fact (in the notation of \cite{CandesTao2005}, $U_j = w$).
$$
\|U_j\|_2 = \|w\|_2 = 
\|\sum_{i = 1}^{+\infty}(-1)^{i-1}w_i\|_2 \le 
\sum_{i = 1}^{+\infty}\|w_i\|_2\le
K(\delta_{3k})\sum_{i = 1}^{+\infty} \Big(\frac{\delta_{3k}}{1-\delta_{3k}}\Big)^{i-1}\le 
2K(\delta_{3k}),
$$
where all inequalities but the last are in \cite[Lemma 2.2]{CandesTao2005}. The last follows from $3\delta_{3k}\le 1.$
\end{proof}

Now, to apply this theorem, we need a matrix satisfying strong RIP-properties. It turns out that this is achieved by random Gaussian matrices. 

\begin{proposition}
\label{prop:candestaogaussianrip}[\cite{CandesTao2005,baraniuk2008simple}]
Suppose that $d\ge Ck\log (n/k)$ for some sufficiently large constant $C$ and let $\delta = 1/6.$ Then, 
with high probability, the matrix $M\in \mathbb{R}^{d\times n}$ in which each entry is independent $\mathcal{N}(0,1/d)$ satisfies
$(1/6,3k)$-RIP.
\end{proposition}
\begin{proof} The proof follows directly from \cite[Lemma 3.1]{CandesTao2005} by choosing the right parameters. Namely, $S = 3k,$ $m = n,$ $p = Ck\log (n/k).$ Then, $r = S/m = 3k/n.$ Hence, for the binary entropy function $H,$
\begin{align*}
& f(r) = \sqrt{n/(Ck\log n)}(\sqrt{3k/n} + \sqrt{2H(3k/n)})\\
& = 
O\Bigg(\sqrt{n/(Ck\log(n/k))} \sqrt{3k/n}\sqrt{\log(n/3k)}\Bigg) = 
O(C^{-1/2}).
\end{align*}
Therefore, for some large enough $C,$ it holds that $f(r)\le 1/128.$ By \cite[Lemma 3.1]{CandesTao2005} with $\epsilon = 1,$
\begin{align*}
    & \prob[1+\delta_{3k}\ge (1 + 2f(r))^2]\le 2\exp(- m H(r)/2)\\
    & = 2\exp(-n \Theta(k\log(n/k)/n)) = 2\exp(-\Theta(k\log(n/k))).
\end{align*}

Altogether, with high probability $1 - \exp(-\Omega(k\log(n/k))),$ $\delta_{3k}\le (1 + 1/64)^2-1\le 1/6.$
\end{proof}

\begin{corollary}
\label{cor:sqrtkfromRIP}
For some large enough $C,$ there exists a margin $\Omega(1/\sqrt{k})$-embedding of $\ksparsen$ in dimension $Ck\log(n/k).$
\end{corollary}
\begin{proof} The proof follows immediately by Theorem~\ref{thm:candestaocertificates} and 
Proposition~\ref{prop:candestaogaussianrip}. Namely, let $\widetilde{V}= (\widetilde{V_1}, \widetilde{V_2}, \ldots, \widetilde{V_n})$ be a random Gaussian matrix of size $d\times n$ with $d = Ck\log(n/k).$ With high probability, the following two properties hold: $\widetilde{V}$ satisfies $(1/6, 3k)$-RIP and each column of $\widetilde{V}$ has norm at most $2$ (apply the RIP with $v = e_i$).

As $\widetilde{V}$ satisfies the $(1/6, 3k)$-RIP, then we know there exists $\{\widetilde{U}_j\}^N_{j = 1}$ such that $\langle \widetilde{U_j}, \widetilde{V_i}\rangle = \frac{1}{\sqrt{k}}$ if $A_{ji} = 1$ and 
$|\langle \widetilde{U_j}, \widetilde{V_i}\rangle|\le \frac{1}{4\sqrt{k}}$ if $A_{ji} = 0.$ This follows by setting $c$ to be the vector $c = \frac{1}{\sqrt{k}}A_{j, :}$ in Theorem~\ref{thm:candestaocertificates}.

Now, let $\bar{V_i} = (\widetilde{V_i}, (5/8)^{1/2}k^{-1/4}))$ and 
$\bar{U_j} = (\widetilde{U_j}, -(5/8)^{1/2}k^{-1/4}).$ We can easily check that $\langle \bar{U_j}, \bar{V_i}\rangle \ge \frac{3}{8\sqrt{k}}$ if $A_{ji} = 1$ and 
$\langle \bar{U_j}, \bar{V_i}\rangle\le -\frac{3}{8\sqrt{k}}$ if $A_{ji} = 0.$
Finally, we know that $\|\bar{V_i}\|_2\le 2, \|\bar{U_j}\|_2 \le C(1/6) = O(1).$ Thus, setting $V_i = \frac{\bar{V_i}}{\|\bar{V_i}\|_2}$ and 
$U_j = \frac{\bar{U_j}}{\|\bar{U_j}\|_2},$ we obtain 
$\langle {U_j}, {V_i}\rangle \ge \Omega(1/\sqrt{k})$ if $A_{ji} = 1$ and 
$\langle {U_j}, {V_i}\rangle\le -\Omega(1/\sqrt{k})$ if $A_{ji} = 0.$
\end{proof}

\begin{remark}
\label{rmk:candestaoinsufficient}
We note that Corollary~\ref{cor:sqrtkfromRIP} falls short of our main results Theorem~\ref{thm:main} and Theorem~\ref{thm:main_construction} in two ways. First, it only provides margin
$\Omega(1/\sqrt{k})$ instead of the optimal $(1+o_k(1))/(2\sqrt{k}).$ It might be possible to fix this by more 
carefully keeping track of the constants. The biggest challenge seems to be the norm of $C(\delta)$ from Theorem~\ref{thm:candestaocertificates}. 

Much more important, however, seems the limitation that in the $k$-sparse setting, only results for margin $\Theta(1/\sqrt{k})$ seem possible from Theorem~\ref{thm:candestaocertificates}. Concretely, 
note that if $c$ is unit norm in Theorem~\ref{thm:candestaocertificates} (which is without loss of generality as the inequalities are homogeneous in $(U,c)$), there exists some $V_i$ such that $\langle U_j, V_i\rangle \le 1/\sqrt{k}.$ Yet, in the current form of the theorem, it might be the case that $\|U_j\|_2\times \|V_i\|_2\ge 1$ as evidenced by the case $V =I$ when $U_j = c_j.$ Thus, $U_j$ needs to be unit norm. Altogether, 
$\langle \frac{U_j}{\|U_j\|_2}, \frac{V_i}{\|V_i\|_2}\rangle \le 1/\sqrt{k}.$ 

To compare, our Theorem~\ref{thm:main} applies also to $k$-sparse $A$ when a larger margin is possible, say $m = 1/k^{1/4}.$ In that case, it will give dimension $O(k^{1/2}\log n).$

At the same time, our Theorem~\ref{thm:main_construction} gives constructions of smaller margin, say $1/n$ but also in much smaller dimensions. Concretely,  $d = \Theta(k^2)$ for margin $1/n.$
\end{remark}

%% file: LowerBoundBeyondnchoosek.tex
\section{Adapting \texorpdfstring{Theorem~\ref{thm:mainlowerbound}}{Theorem 1.2} to general 
\texorpdfstring{$k$}{k}-sparse matrices}
\label{sec:LowerBoundBeyondnchoosek}

We note that the same proof method of Theorem \ref{thm:mainlowerbound} can be applied to an arbitrary relevance matrix. We first need a modified version of Lemma \ref{lem:reductiontosparselowerbound}. Towards this end, we need the following definition. Let $\mathcal{S}$ be a family of subsets of $[n].$ We say that $B\subseteq [n]$ is shattered if for any $C\subseteq B,$ there exists some $S\in \mathcal{S}$ such that $S\cap B = C.$ 

\begin{lemma}\label{lem:reductiontoVCdimension}
Let $A\in \{0,1\}^{N\times n}$, and let $\{V_i\}_{i=1}^n,\{U_j\}_{j=1}^N \in \mathbb{R}^d$ be a margin-$m$ embedding of $A$. Define the linear map $V:\mathbb{R}^n\to \mathbb{R}^d$ by $Ve_i=V_i$. Suppose that $S\subseteq [n]$ is shattered by the family of row supports
\[
\mathcal{C}_A \coloneqq \bigl\{\{i\in[n]:A_{ji}=1\}:j\in[N]\bigr\}.
\]
Then for any $x\in \mathbb{R}^n$ supported on $S$, it holds that $\|Vx\|_2 \ge m\|x\|_1$.
\end{lemma}

\begin{proof}
By homogeneity, it is enough to prove the claim when $\|x\|_1=1$. Since $S$ is shattered, there exists some $j\in [N]$ such that for every $i\in S$, $A_{ji}=1 \iff x_i>0$. Therefore,
$$
\left\langle U_j,\sum_{i\in S}x_iV_i\right\rangle
=
\sum_{i\in S}x_i\langle U_j,V_i\rangle
\ge
m\sum_{i\in S}|x_i|
=
m.
$$
Indeed, if $x_i>0$, then $A_{ji}=1$, so $\langle U_j,V_i\rangle\ge m$, while if $x_i<0$, then $A_{ji}=0$, so $\langle U_j,V_i\rangle\le -m$, and hence again $x_i\langle U_j,V_i\rangle \ge m|x_i|$. Since $\|U_j\|_2=1$, it follows that
$$
\|Vx\|_2
=
\Big\|\sum_{i\in S}x_iV_i\Big\|_2
\ge
\Big\langle U_j,\sum_{i\in S}x_iV_i\Big\rangle
\ge m.\qedhere
$$
\end{proof}

Then we can use a nearly identical proof as in Theorem \ref{thm:mainlowerbound}.

\begin{theorem}[Lower bound for $k$-sparse matrices]\label{rem:ksparsegeneralization}
Let $A\in\{0,1\}^{N\times n}$ be such that every row of $A$ has exactly $k$ ones, and let
$$\mathcal{C}_A \coloneqq \bigl\{\{i\in[n]:A_{ji}=1\}: j\in[N]\bigr\}\subseteq 2^{[n]}
$$
denote the family of row supports. Suppose $\{V_i\}_{i=1}^n,\{U_j\}_{j=1}^N\subset\mathbb{R}^d$ form a margin-$m$ embedding of $A$ with $m > 0$. Fix $\alpha\in(0,1/2]$ and set $s=\lfloor \alpha k\rfloor$. If $s\ge 1$, then
$$
d \ge \frac{\log |\mathcal{F}_s(A)|}{\log\!\left(1+\frac{2}{m\sqrt{s}}\right)},
$$
where $\mathcal{F}_s(A)\subseteq \binom{[n]}{s}$ is a family such that for all distinct $T,T'\in\mathcal{F}_s(A)$, $|T\cap T'|<\frac{s}{2}$ and $T\cup T'$ is shattered by $\mathcal{C}_A$.
\end{theorem}

\begin{proof}[Proof of Theorem~\ref{rem:ksparsegeneralization}]
Let $V$ be as in Lemma \ref{lem:reductiontoVCdimension}. Set $s=\lfloor \alpha k\rfloor$. By definition, there exists a family $\mathcal{F}_s(A) \subseteq \binom{[n]}{s}$ such that for any distinct $T,T' \in \mathcal{F}_s(A)$, we have $|T\cap T'| < \frac{s}{2}$ and $T\cup T'$ is shattered by $\mathcal{C}_A$.

For each $T \in \mathcal{F}_s(A)$, define the sign vector $\sigma^T \in \{\pm 1\}^T$ so that $\left\|\sum_{i \in T} \sigma_i^T V_i\right\|_2$ is minimized, and set $y_T \coloneqq \sum_{i \in T} \sigma_i^T V_i$.

To bound $\|y_T\|_2$, let $\epsilon$ be uniformly random in $\{\pm 1\}^T$. Then

$$\mathbb{E}\left\|\sum_{i\in T}\epsilon_i V_i\right\|_2^2
=
\mathbb{E}\left[
\sum_{i\in T}\|V_i\|_2^2
+
\sum_{\substack{i,\ell\in T\\ i\neq \ell}}
\epsilon_i\epsilon_\ell \langle V_i,V_\ell\rangle
\right]
=
\sum_{i\in T}\|V_i\|_2^2
=
s,
$$
since the cross terms vanish in expectation. Because $\sigma^T$ minimizes the norm, it follows that
\[
\|y_T\|_2^2
=
\Bigl\|\sum_{i\in T}\sigma_i^T V_i\Bigr\|_2^2
\le s.
\]
Hence $y_T \in B_d(0,\sqrt{s})$ for all $T \in \mathcal{F}_s(A)$.

Now define $x_T \coloneqq \sum_{i \in T}\sigma_i^T e_i$ so that $Vx_T = y_T$. Consider distinct $T,T' \in \mathcal{F}_s(A)$. Since $\operatorname{supp}(x_T-x_{T'}) \subseteq T\cup T'$ and $T\cup T'$ is shattered by $\mathcal{C}_A$, Lemma \ref{lem:reductiontoVCdimension} gives
$$
\|V(x_T-x_{T'})\|_2 \ge m\|x_T-x_{T'}\|_1.
$$
Moreover, since every coordinate of $x_T$ and $x_{T'}$ lies in $\{0,\pm 1\}$,
$$
\|x_T-x_{T'}\|_1 \ge |T\triangle T'|
= |T|+|T'|-2|T\cap T'|
> 2s-s
= s.
$$
Therefore,
$$
\|y_T-y_{T'}\|_2
=
\|Vx_T - Vx_{T'}\|_2
=
\|V(x_T-x_{T'})\|_2
\ge m\|x_T-x_{T'}\|_1
\ge ms.
$$

Thus, the vectors $\{y_T\}_{T\in \mathcal{F}_s(A)}$ are pairwise $ms$-separated in the ball $B_d(0,\sqrt{s})$. When we rescale by $\tilde y_T \coloneqq y_T/\sqrt{s}$, we obtain a collection of points in $B_d(0,1)$ that are pairwise $m\sqrt{s}$-separated. By the standard volume bound (for example \cite[Corollary 4.2.11]{Vershynin2026}),
$$
|\mathcal{F}_s(A)| \le \left(1+\frac{2}{m\sqrt{s}}\right)^d.
$$
So, it immediately follows that when $s=\lfloor \alpha k\rfloor$, we have
$$
d \geq \frac{\log |\mathcal{F}_{\lfloor \alpha k\rfloor}(A)|}{\log\left(1+\frac{2}{m\sqrt{\lfloor \alpha k\rfloor}}\right)}$$
which concludes the proof.
\end{proof}

%% file: ProofsofMarginImportance.tex
\section{Proofs of Propositions on Margin Importance}
\label{sec:margin-importance}

\begin{proposition}[Compositional Generalization and Margin]
\label{prop:margincompositional}
Suppose $\{U_j\}_{j = 1}^N$, $\{V_i\}_{i = 1}^n$ form a margin-$m$ embedding in $\mathbb{R}^d$ of $\ksparsen$ with $0 < m < 1$. Then for all $T \subset [n]$ with $k \leq |T| \leq \min(k + \frac{2km}{1 -m} ,n - 1)$, there exists an $h \in \mathbb{S}^{d-1}$ and $c \in \mathbb{R}$ such that for all $i\in T,$
it holds that 
$\langle h, V_i\rangle \geq c$ while for all $j \not \in T,$
$\langle h, V_j\rangle \leq c.$
\end{proposition}

\begin{proof}[Proof of Proposition \ref{prop:margincompositional}]

For a set $S \subseteq [n]$ with $|S| = k$, let $U_S$ denote the vector such that $\langle U_S, V_i\rangle \geq m$ for all $i \in S$ and $\langle U_S, V_j\rangle \leq -m$ for all $j \notin S$. Define
$$
h_T \coloneqq \sum_{S \subseteq T,\ |S| = k} U_S.
$$
Then
\begin{align*}
\langle h_T, V_i\rangle
&=
\left\langle
\sum_{S \subseteq T,\ i \in S,\ |S| = k} U_S
+
\sum_{S \subseteq T,\ i \notin S,\ |S| = k} U_S,
V_i
\right\rangle\\
&=
\sum_{S \subseteq T,\ i \in S,\ |S| = k} \langle U_S, V_i\rangle
+
\sum_{S \subseteq T,\ i \notin S,\ |S| = k} \langle U_S, V_i\rangle.
\end{align*}

Notice that if $i \in S$, then by definition $\langle U_S, V_i\rangle \geq m$. If $i \notin S$, then since $U_S$ and $V_i$ have unit norm, we have $\langle U_S, V_i\rangle \geq -1$.

Now fix any $i \in T$. Then there are exactly
$\binom{|T|-1}{k-1}$ sets of size $k$ containing $i$ and contained in $T$. Also, there are exactly $\binom{|T|}{k} - \binom{|T|-1}{k-1}$ sets of size $k$ contained in $T$ that do not contain $i$.

So it follows that
\begin{align*}
\langle h_T, V_i\rangle
&=
\sum_{S \subseteq T,\ i \in S,\ |S| = k} \langle U_S, V_i\rangle
+
\sum_{S \subseteq T,\ i \notin S,\ |S| = k} \langle U_S, V_i\rangle \\
&\geq
m \binom{|T|-1}{k-1}
+
(-1)\left( \binom{|T|}{k} - \binom{|T|-1}{k-1}\right) \\
&=
(1+m)\binom{|T|-1}{k-1} - \binom{|T|}{k}\\
&=
\frac{(1+m)k}{|T|} \binom{|T|}{k} -  \binom{|T|}{k}\\
&=
\frac{(1+m)k - |T|}{|T|} \binom{|T|}{k}.
\end{align*}

Furthermore, if $i \notin T$, then $\langle U_S, V_i\rangle \leq -m$ for all $S \subseteq T$ with $|S| = k$. Since there are $\binom{|T|}{k}$ subsets of size $k$ contained in $T$, we get
$$
\langle h_T, V_i\rangle
=
\sum_{S \subseteq T,\ |S| = k} \langle U_S, V_i\rangle
\leq -m \binom{|T|}{k}.
$$
We also know that, since $m < 1$,
\begin{align*}
k + \frac{2km}{1-m} \geq |T|
&\implies (1-m)k + 2mk \geq |T|(1-m) \\
&\implies (1+m)k - |T| \geq -m|T| \\
&\implies \frac{(1+m)k - |T|}{|T|} \geq -m.
\end{align*}

Therefore, for every $i \in T$ and $j \notin T$, we have
$$
\langle h_T, V_i\rangle
\geq
\frac{(1+m)k - |T|}{|T|}\binom{|T|}{k}
\geq
-m\binom{|T|}{k}
\geq
\langle h_T, V_j\rangle.
$$
Therefore, there exists a real constant $c' \in \mathbb{R}$ satisfying
$$
\langle h_T, V_i\rangle \geq c' \geq \langle h_T, V_j\rangle
\qquad
\text{for all } i \in T \text{ and } j \notin T.
$$

Notice that $\|h_T\| \neq 0$ because, since $|T| < n$, for some $i \notin T$ we have $\langle V_i, h_T \rangle \leq -m\binom{|T|}{k} < 0$. So, if we set $h\coloneqq h_T/\|h_T\|$, it follows that there exists an $h \in \mathbb{S}^{d-1}$ and $c = c'/\|h_T\|\in \mathbb{R}$ such that 
$$\langle h, V_i\rangle \geq c \qquad \forall i \in T, \quad \text{and}\quad \langle h, V_j\rangle \leq c \qquad \forall j \notin T.\qedhere$$
\end{proof}

We note that even a small margin already suffices to force downward generalization. Indeed, once $m \geq \frac{k -1 }{2n - k - 1}$, the following proposition implies that every nonempty subset with $1 \leq |T| \leq k$ is retrievable. This is related to the notion of a $k$-neighborly polytope \cite{grunbaum1967convex}: if $\{V_i\}_{i = 1}^n$ forms a $k$-neighborly polytope, then every subset of at most $k$ vertices forms a face and is therefore exposed by a supporting hyperplane. Thus, every such subset can be separated from the remaining vertices and hence can be retrieved.

\begin{proposition}[Downward Generalization and Margin]
\label{prop:margindownward}
Suppose $\{U_j\}_{j = 1}^N$, $\{V_i\}_{i = 1}^n$ form a margin-$m$ embedding in $\mathbb{R}^d$ of $\ksparsen$ with $0 < m < 1$. Then for all $T \subset [n]$ with $\max\{1, k - \frac{2m(n-k)}{1-m}\} \leq |T| \leq k$, there exists an $h \in \mathbb{S}^{d-1}$ and $c \in \mathbb{R}$ such that for all $i\in T,$
it holds that 
$\langle h, V_i\rangle \geq c$ while for all $j \not \in T,$
$\langle h, V_j\rangle \leq c.$
\end{proposition}

\begin{proof}[Proof of Proposition \ref{prop:margindownward}]
The case $|T|=k$ is immediate by taking $h=U_T$ and $c=0$, so assume $|T|<k$.

We use the same notation as in the proof of Proposition \ref{prop:margincompositional}. Define
$$
h_T \coloneqq \sum_{S\supset T,\ |S| = k} U_S.
$$
Then
\begin{align*}
\langle h_T, V_i\rangle
&=
\left\langle
\sum_{S\supset T,\ i \in S,\ |S| = k} U_S
+
\sum_{S\supset T,\ i \notin S,\ |S| = k} U_S,
V_i
\right\rangle\\
&=
\sum_{S\supset T,\ i \in S,\ |S| = k} \langle U_S, V_i\rangle
+
\sum_{S\supset T,\ i \notin S,\ |S| = k} \langle U_S, V_i\rangle.
\end{align*}

Again, if $i \in S$, then by definition $\langle U_S, V_i\rangle \geq m$.

Now fix any $i \in T$. Then there are exactly
$\binom{n - |T|}{k - |T|}$ sets of size $k$ containing $T$, and so 
$$
\langle h_T, V_i\rangle \geq m \binom{n - |T|}{k - |T|}.
$$

For $j \notin T$, the number of sets of size $k$ containing $T$ and $j$ is $\binom{n - |T| - 1}{k - |T| - 1}$, while the number of sets of size $k$ containing $T$ but not $j$ is $\binom{n - |T| - 1}{k - |T|}$. Hence, using $\langle U_S,V_j\rangle \leq 1$ when $j\in S$ and $\langle U_S,V_j\rangle \leq -m$ when $j\notin S$, we have
$$
\langle h_T,V_j\rangle
\leq
\binom{n-|T|-1}{k-|T|-1}
-
m\binom{n-|T|-1}{k-|T|}.
$$
Thus it suffices to show
$$
\binom{n-|T|-1}{k-|T|-1}- m\binom{n-|T|-1}{k-|T|}
\leq
m\binom{n-|T|}{k-|T|}.
$$
Since
$$
\binom{n-|T|}{k-|T|} = \binom{n-|T|-1}{k-|T|-1} + \binom{n-|T|-1}{k-|T|},
$$
this inequality is equivalent to
$$
m \geq
    \frac{\binom{n-|T|-1}{k-|T|-1}}
    {\binom{n-|T|-1}{k-|T|-1}+2\binom{n-|T|-1}{k-|T|}}
    =
    \frac{k-|T|}{k-|T|+2(n-k)}
    \iff
    |T| \geq k - \frac{2m(n-k)}{1-m}.
$$

Therefore, for every $i \in T$ and $j \notin T$, 
$$
\langle h_T, V_i\rangle
\geq
c'\coloneqq m \binom{n - |T|}{k - |T|}
\geq
\langle h_T, V_j\rangle.
$$
Since $c'>0$, we have $h_T\neq 0$. Taking $h=\frac{h_T}{\|h_T\|}$ and $c=\frac{c'}{\|h_T\|}
$ completes the proof.
\end{proof}

\begin{proposition}[Robustness and Margin]
\label{prop:marginrobustness}
Suppose that the object space $\mathcal{O}$ and the query space $\mathcal{Q}$ are metric spaces with metrics $d_{\mathcal{O}}$ and $d_{\mathcal{Q}},$ respectively. Suppose that the corresponding encoders $f_\theta$ and $g_\phi$ have relative-bias-$\tau$ and margin $m$ on the datasets $\{Q_j\}_{j=1}^N \subset \mathcal{Q}$ and $ \{O_i\}_{i=1}^n \subset \mathcal{O}$ with relevance matrix $A\in \{0,1\}^{N\times n}$.  If $f_\theta$ and $g_\phi$ are both $L$-Lipschitz for some $L > 0$, then for every $j\in[N]$, every perturbation $Q_j' \in \mathcal{Q}$ satisfying
$
d_{\mathcal{Q}}(Q_j',Q_j)<\frac{m}{L}
$
preserves perfect retrieval. Namely,
$
\langle g_\phi(Q_j'), f_\theta(O_i)\rangle > \tau \quad \text { if } A_{ji}=1
$ and 
$
\langle g_\phi(Q_j'), f_\theta(O_i)\rangle < \tau \quad \text{if } A_{ji}=0.
$ Likewise, the same holds for small perturbations of the objects.
\end{proposition}
\begin{proof}[Proof of Proposition \ref{prop:marginrobustness}]

Since $d_{\mathcal{Q}}(Q_j',Q_j)< \frac{m}{L}$ and $g_\phi$ is $L$-Lipschitz, we have
\[
\|g_\phi(Q_j')-g_\phi(Q_j)\|
\leq L\, d_{\mathcal{Q}}(Q_j',Q_j)
< m.
\]

Now fix $i\in[n]$. If $A_{ji}=1$, then since $f_\theta(O_i)$ has norm $1$, we have
\[
\bigl|\langle g_\phi(Q_j')-g_\phi(Q_j),\, f_\theta(O_i)\rangle\bigr|
\leq \|g_\phi(Q_j')-g_\phi(Q_j)\|\,\|f_\theta(O_i)\|
< m.
\]
Therefore,
\begin{align*}
\langle g_\phi(Q_j'), f_\theta(O_i)\rangle
&=
\langle g_\phi(Q_j')-g_\phi(Q_j), f_\theta(O_i)\rangle
+
\langle g_\phi(Q_j), f_\theta(O_i)\rangle \\
&> -m + (\tau + m) = \tau.
\end{align*}

Similarly, if $A_{ji}=0$, then
\begin{align*}
\langle g_\phi(Q_j'), f_\theta(O_i)\rangle
&=
\langle g_\phi(Q_j')-g_\phi(Q_j), f_\theta(O_i)\rangle
+
\langle g_\phi(Q_j), f_\theta(O_i)\rangle \\
&< m + (\tau - m) = \tau.\qedhere
\end{align*}
\end{proof}

\begin{proposition}[Maximal-Margin Embeddings and Sigmoid Loss]
\label{prop:marginsigmoid} 
Consider the sigmoid loss with inverse-temperature $t>0$ and bias $b$ over embeddings $\{U_j\}_{j = 1}^N, \{V_i\}_{i = 1}^n:$ 
$$
\sigloss(\{U_j\}_{j = 1}^N, \{V_i\}_{i = 1}^n; t, b)\coloneqq \sum_{j\in [N], i \in [n]} \log\Big(1 + \exp\big((-1)^{A_{ji}}(t\langle U_j, V_i\rangle -b)\big)\Big).
$$
Then
\begin{enumerate}
    \item If $\sigloss(\{U_j\}_{j = 1}^N, \{V_i\}_{i = 1}^n; t, b) < \log 2,$ then there exists $m > 0$ such that $\{U_j\}_{j = 1}^N, \{V_i\}_{i = 1}^n$ is a margin-$m,$ relative-bias-$\frac{b}{t}$ embedding of $A$.
    \item If $\{U_j\}_{j=1}^N$ and $\{V_i\}_{i=1}^n$ form a margin-$m_*$, relative-bias-$\tau$ embedding of $A$ where
$$
m_* \coloneqq \min\left\{
\min_{A_{ji}=1}\left(\langle U_j,V_i\rangle-\tau\right),
\min_{A_{ji}=0}\bigl(\tau-\langle U_j,V_i\rangle\bigr)
\right\} > 0.
$$
Then
$$
\lim_{T\to\infty}\frac{1}{T}\log \sigloss(\{U_j\}_{j=1}^N,\{V_i\}_{i=1}^n;T,T\tau)=-m_*.
$$
\end{enumerate}
\end{proposition}

\begin{proof}[Proof of Proposition \ref{prop:marginsigmoid}]

We begin with part (1). 
Suppose that $\sigloss(\{U_j\}_{j=1}^N,\{V_i\}_{i=1}^n;t,b)<\log 2$. For each fixed pair $(j,i)$, the term $\log\Big(1 + \exp\big((-1)^{A_{ji}}(t\langle U_j, V_i\rangle -b)\big)\Big)$ is nonnegative. So, it follows that for every fixed pair $(j,i),$
$$
\log\Bigl(1+\exp\bigl((-1)^{A_{ji}}(t\langle U_j,V_i\rangle-b)\bigr)\Bigr)<\log 2.
$$
Since the function $x\mapsto \log(1+e^x)$ is strictly increasing, we obtain
$$
(-1)^{A_{ji}}(t\langle U_j,V_i\rangle-b)<0.
$$
If $A_{ji}=1$, then $-(t\langle U_j,V_i\rangle-b)<0$, so $t\langle U_j,V_i\rangle-b>0.$ If $A_{ji}=0$, then $t\langle U_j,V_i\rangle-b<0.$

Since there are finitely many pairs, set
$$
m\coloneqq
\min\left\{
\min_{A_{ji}=1}\left(\langle U_j,V_i\rangle-\frac {b}{t}\right),
\min_{A_{ji}=0}\left(\frac {b}{t}-\langle U_j,V_i\rangle\right)
\right\}>0.
$$
Then $\{U_j\}_{j=1}^N,\{V_i\}_{i=1}^n$ form a margin-$m$, relative-bias-$\frac {b}{t}$ embedding of $A$.


We now prove part 2. For each pair $(j,i) \in [N] \times [n]$, define
$$
\gamma_{ji}\coloneqq
\begin{cases}
\langle U_j,V_i\rangle-\tau, & A_{ji}=1,\\[4pt]
\tau-\langle U_j,V_i\rangle, & A_{ji}=0.
\end{cases}
$$
Then $\gamma_{ji}>0$ for all $(j,i)$, and by definition $m_*=\min_{j,i}\gamma_{ji}$.

If $A_{ji}=1$, then
$$
(-1)^{A_{ji}}\bigl(T\langle U_j,V_i\rangle-T\tau\bigr)
=
-\bigl(T\langle U_j,V_i\rangle-T\tau\bigr)
=
-T(\langle U_j,V_i\rangle-\tau)
=
-T\gamma_{ji}.
$$
If $A_{ji}=0$, then
$$
(-1)^{A_{ji}}\bigl(T\langle U_j,V_i\rangle-T\tau\bigr)
=
T\langle U_j,V_i\rangle-T\tau
=
-T(\tau-\langle U_j,V_i\rangle)
=
-T\gamma_{ji}.
$$
Therefore,
$$
\sigloss(\{U_j\}_{j=1}^N,\{V_i\}_{i=1}^n;T,T\tau)
=
\sum_{j\in[N],\,i\in[n]}\log\!\bigl(1+e^{-T\gamma_{ji}}\bigr).
$$

Since $\gamma_{ji}\ge m_*$ for all $(i,j)$, we obtain the upper bound
$$\sigloss(\{U_j\}_{j = 1}^N,\{V_i\}_{i = 1}^n;T,T\tau)
\leq
Nn\log(1+e^{-Tm_*})
\leq
Nne^{-Tm_*},
$$
where we used $\log(1+x)\le x$ for all $x\ge0$. For the lower bound, choose $(j_0,i_0)$ such that $\gamma_{j_0i_0}=m_*$. Then $\sigloss(\{U_j\}_{j=1}^N,\{V_i\}_{i = 1}^n;T,T\tau)
\ge
\log(1+e^{-Tm_*})$. Also, for $x\in[0,1]$, $\log(1+x)\ge \frac{x}{2}$. Hence for all sufficiently large $T$ (so that $e^{-Tm_*}\le1$), $\log(1+e^{-Tm_*})\ge \frac12 e^{-Tm_*}$. Thus for all sufficiently large $T$,
$$
\frac{1}{2} e^{-Tm_*}
\leq
\sigloss(\{U_j\}_{j = 1}^N,\{V_i\}_{i = 1}^n;T,T\tau)
\le
Nne^{-Tm_*}.
$$

Taking logarithms gives
$$
-Tm_*-\log 2
\le
\log \sigloss(\{U_j\}_{j = 1}^N,\{V_i\}_{i = 1}^n;T,T\tau)
\le
-Tm_*+\log(Nn).
$$
Dividing by $T$ and letting $T\to\infty$, we conclude that
$$
\lim_{T\to\infty}
\frac{1}{T}\log \sigloss(\{U_j\}_{j = 1}^N,\{V_i\}_{i = 1}^n;T,T\tau)
=
-m_*\qedhere
$$
\end{proof}

\begin{proposition}[Large Margins and Quantization]
\label{prop:marginquantization}
Suppose that $\{U_j\}_{j = 1}^N$, $\{V_i\}_{i = 1}^n$ is a margin-$m$ embedding of $A$. Then there exists a codebook $\mathcal{C} \subset \mathbb{S}^{d-1}$ with size $|\mathcal{C}| \leq \left(1 + \frac{8}{m}\right)^d$ such that the quantized embeddings $\tilde{U}_j \in \arg\min_{x \in \mathcal{C}} \|x - U_j\|$ and $\tilde{V}_i \in \arg\min_{x \in \mathcal{C}} \|x - V_i\|$ form a margin-$\frac{m}{2}$ embedding.
\end{proposition}

\begin{proof}
    There exists a $\frac{m}{4}$-net $\mathcal{C}$ of size $|\mathcal{C}| \leq \left(1 + \frac{2}{\frac{m}{4}}\right)^d = \left(1 + \frac{8}{m}\right)^d$ \cite[Corollary 4.2.11] {Vershynin2026}. For any $j \in [N]$ and $i \in [n]$, we have

    $$|\langle \tilde{U}_j, \tilde{V}_i\rangle - \langle U_j, V_i\rangle| =  |\langle \tilde{U}_j - U_j, \tilde{V}_i\rangle + \langle U_j,  \tilde{V}_i - V_i\rangle|$$

    Since $\|\tilde{V}_i\| = \|U_j\| = 1$, we know that $|\langle \tilde{U}_j, \tilde{V}_i\rangle - \langle U_j, V_i\rangle| \leq \frac{m}{2}$ because $\|\tilde{U}_j - U_j\|, \|\tilde{V}_i - V_i\| \leq \frac{m}{4}$. Therefore, the vectors $\{\tilde{U}_j\}_{j = 1}^N$, $\{\tilde{V}_i\}_{i = 1}^n$ form a margin-$\frac{m}{2}$ embedding of $A.$ 
\end{proof}

%% file: FreeEmbedding.tex
\section{InfoNCE and Sigmoid in The Free Embedding Model}
\label{sec:freeembedding}
\subsection{InfoNCE loss}
Recall from Proposition~\ref{prop:marginsigmoid} that the global minimizers of the following sigmoid loss
$$
\sigloss(\{U_j\}_{j = 1}^N, \{V_i\}_{i = 1}^n; t, b)\coloneqq \sum_{j\in [N], i \in [n]} \log\Big(1 + \exp\big((-1)^{A_{ji}}(t\langle U_j, V_i\rangle -b)\big)\Big).
$$
are exactly the embeddings of $A$ as in Definition~\ref{def:maindefinition}.

The InfoNCE, on the other hand, defined as
\begin{equation}
\label{eq:definfonce}
\infonceloss(\{U_j\}_{j = 1}^N, \{V_i\}_{i = 1}^n; t)\coloneqq - \sum_{(j,i) \in [N]\times [n]\; : \; A_{ji} = 1}
\log\Big(
\frac{\exp(t\langle U_j, V_i\rangle)}{\sum_\ell \exp(t\langle U_j, V_\ell\rangle)}
\Big).
\end{equation}
When $A = I,$ the recent work \cite{bangachev2025globalminimizerssigmoidcontrastive} shows that the global minimizers of InfoNCE have a similar geometry to Definition~\ref{def:maindefinition}, except that there is a different relative bias $\tau_j$ corresponding to each query. Again, they show that the loss converges to 0 as $T\longrightarrow +\infty.$

We note, however, that when $A$ has rows of sparsity more than 1, the loss cannot converge to 0. In fact, for some embeddings of $A$ satisfying the conditions in Definition~\ref{def:maindefinition}, it is possible that $\infonceloss$ diverges as $T\longrightarrow+\infty.$

\begin{proposition}
    Suppose that $A\in \{0,1\}^{N\times n}$ is such that $A_{ji_1} = A_{ji_2} = 1$ and $\{U_j\}_{j = 1}^N, \{V_i\}_{i = 1}^n$ is an embedding of $A$ such that $\langle U_j,V_{i_1}\rangle \neq  \langle U_j,V_{i_2}\rangle.$ Then, 
    $\lim_{T\longrightarrow + \infty}
    \infonceloss(\{U_j\}_{j = 1}^N, \{V_i\}_{i = 1}^n; T) = +\infty.
    $
\end{proposition}
\begin{proof} First note that each summand in \eqref{eq:definfonce} is non-negative. Thus,
\begin{align*}
    & \infonceloss(\{U_j\}_{j = 1}^N, \{V_i\}_{i = 1}^n; T)\\
    & \ge \log\Big(
\frac{\sum_\ell \exp(T\langle U_j, V_\ell\rangle)}{\exp(T\langle U_j, V_{i_1}\rangle)}
\Big)+
\log\Big(
\frac{\sum_\ell \exp(T\langle U_j, V_\ell\rangle)}{\exp(T\langle U_j, V_{i_2}\rangle)}
\Big)\\
& \ge \log
\frac{\exp(T\langle U_j, V_{i_1}\rangle) + \exp(T\langle U_j, V_{i_2} \rangle)}{\exp(T\langle U_j, V_{i_1}\rangle)} + 
\log
\frac{\exp(T\langle U_j, V_{i_1}\rangle) + \exp(T\langle U_j, V_{i_2} \rangle)}{\exp(T\langle U_j, V_{i_2}\rangle)}\\
& \ge 
\log
\frac{\exp(T\langle U_j, V_{i_1}\rangle) + \exp(T\langle U_j, V_{i_2} \rangle)}{\min(\exp(T\langle U_j, V_{i_1}\rangle),\exp(T\langle U_j, V_{i_2}\rangle)) }\\
& = 
\log\bigg(1 + \exp\Big(T |\langle U_j, V_{i_1}\rangle - 
\langle U_j, V_{i_2}\rangle|\Big)\bigg)
\end{align*}
Taking $T\longrightarrow +\infty$ shows that
$\infonceloss(\{U_j\}_{j = 1}^N, \{V_i\}_{i = 1}^n; T)\longrightarrow + \infty.$
\end{proof}

\subsection{Experiment}
We choose to work with $\ksparsen$ for $k \in \{2,3\}$ and 
$n \in \{20,40,60,\ldots, 240\}$ when $k = 2, $
$n \in \{20,40,60,\ldots, 120\}$ when $k = 3.$

We set the inverse temperature to be trainable for both InfoNCE and sigmoid, as common in practice, e.g. \cite{zhai23siglip}, and we fix the bias to $b=0$ for sigmoid loss. We initialize the embeddings $\{U\}_{j = 1}^N,\{V\}_{i = 1}^n$ as random i.i.d. points on $\mathbb{S}^{d-1}.$ Then, we run the Adam optimizer with cosine annealing learning rate over $\sigloss, \infonceloss$ for 100000 steps to obtain the embeddings. We embed in each of the dimensions $d\in \{5,6,7,\ldots, 30\}.$ We start from 5 since this is the minimal dimension for which an embedding of $\ksparsen$ with $k = 2$ exists, see Theorem~\ref{thm:signrankforsparse}.

We ran our experiments on a single H200 GPU. The runs for each triplet $n,d,k$ took up to one hour (for our largest experiment, $n = 120, k = 3, d = 30$).

Then, we plot in Figure~\ref{fig:min-nonzero-margin} the smallest dimension for which a positive margin was achieved when $k = 2$ and $n$ varies. This smallest dimension is an effective sign rank for the respective loss.

In Figure~\ref{fig:max-margin-vs-logn} we choose three dimensions and plot the respective margins obtained when $k = 2$ for the varying $n.$ We reproduce the same results with $k = 3$ and obtain qualitatively similar behavior:

\begin{figure}[htb!]
    \centering
    \begin{minipage}[t]{0.47\textwidth}
        \centering
        \includegraphics[width=\linewidth]{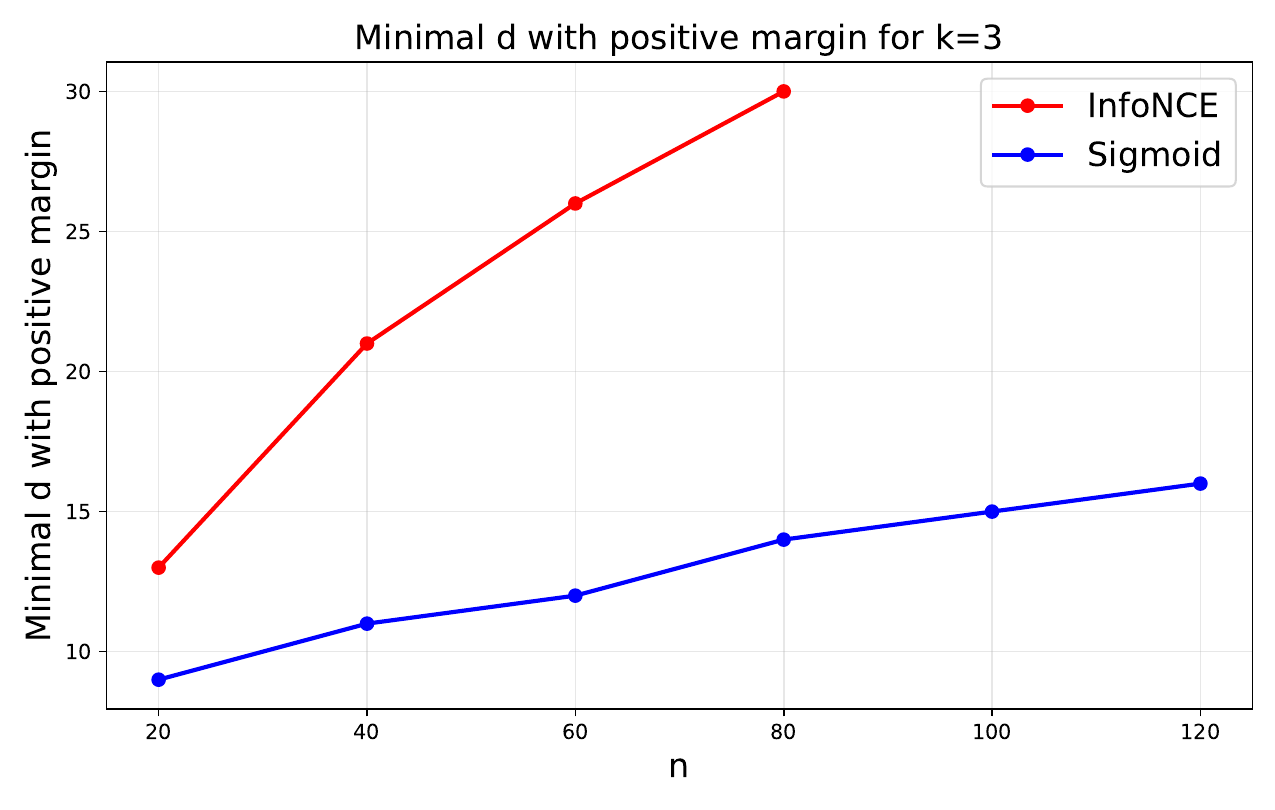}
        \caption{\small{Minimal dimension needed to achieve a non-zero margin after 100000 training steps for the InfoNCE and sigmoid losses. The sigmoid succeeds in much smaller dimensions.}}
        \label{fig:min-nonzero-margin3}
    \end{minipage}
    \hfill
    \begin{minipage}[t]{0.51\textwidth}
        \centering
        \includegraphics[width=.95\linewidth]{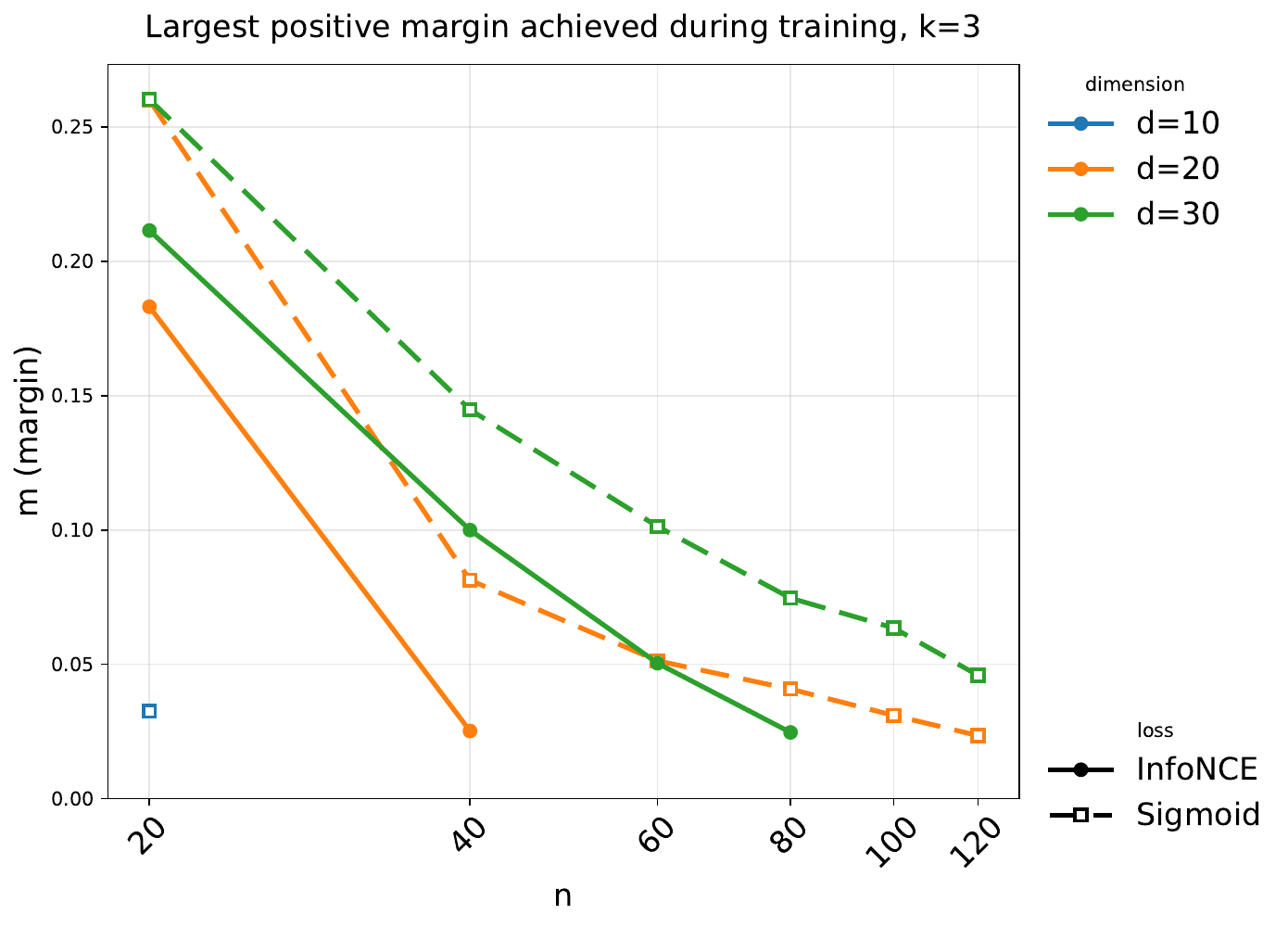}
        \caption{\small{Maximal margin achieved during 100000 training steps for $d \in \{10,20,30\}$, using the InfoNCE and sigmoid losses showing clear advantage of the sigmoid loss.}}
        \label{fig:max-margin-vs-logn3}
    \end{minipage}
\end{figure}

%% file: LLMs.tex
\section{LLM Usage}
\label{sec:llmusage}
We used LLMs for completing some proofs, finding related works, generating code, and editing. We wrote the initial manuscript without any LLM usage and only then used LLMs for editing. We outline the theorem-proving process in more detail, as it could be valuable to the community.

\textbf{Proofs:} We used ChatGPT Pro 5.4 - Extended for proving some theorems. Here is the breakdown of the process, with prompts. We format the prompts in LaTeX and clean up typos, but otherwise there are no changes.     
\begin{enumerate}
    \item Corollary~\ref{cor:sqrtkfromRIP}: We wanted to see if we could find a matching construction for the $O(k\log(n/k)/\log(k))$ lower bound on the dimension necessary for $\Theta(k^{-1/2})$ margin from \cite{weller2025theoretical}. We prompted ChatGPT with the following prompt to obtain Corollary~\ref{cor:sqrtkfromRIP}:
    \medskip
    
    \fbox{
    \begin{minipage}{33em}
    Hi! I want a construction for the following problem. Given is the following bipartite graph parametrized by integers $n$ and $k.$ On the left side, there are n vertices. On the right side, there are $N = \binom{n}{k}$ corresponding to the subsets of $[n]$ of size $k$. For each subset of $k$ vertices on the left side, there is a unique vertex on the right side adjacent exactly to them. Now, I want to produce n unit vectors in dimension $d$ $V_1, V_2, \ldots  , V_n$ for the vertices on the left and $N$ unit vectors in dimension $d$ $U_1, U_2, \ldots  U_N$ for the vertices on the right such that $\langle V_i, U_j\rangle > \Omega(1/\sqrt{k})$ if and only if $(i,j)$ is an edge and $\langle V_i, U_j\rangle <- \Omega(1/\sqrt{k})$ otherwise. Give me such a construction in dimension $d= O(k\log n).$
    \end{minipage}
    }
    
    \medskip
    
    \item Theorem~\ref{thm:main}: We knew how to prove the weaker result for $d = O(\maxmargin(+\infty, A)^{-2}\log(nN))$ using the Johnson-Lindenstrauss lemma. As stated, however, this weaker result only implies dimension $O(k^2\log n)$ for margin $\Theta(k^{-1/2})$-embeddings of $\ksparsen.$ We already knew that it is possible to improve this to $d= O(k\log n)$ via Corollary~\ref{cor:sqrtkfromRIP}. We conjectured that Theorem~\ref{thm:mainlowerbound} should be true. We spent some time devising a proof strategy and submitted the following sequence of prompts to ChatGPT. The first one is to check if it can come up with the result we already had:
        
    \medskip
    
\fbox{
    \begin{minipage}{33em}
    You are an expert researcher at graph theory, random matrix theory, compressed sensing! Consider the following problem. 
    
   Given is the following bipartite graph parametrized by integers $n$ and $k.$ On the left side, there are n vertices. On the right side, there are $N = \binom{n}{k}$ corresponding to the subsets of $[n]$ of size $k$. For each subset of $k$ vertices on the left side, there is a unique vertex on the right side adjacent exactly to them. Now, I want to produce n unit vectors $V_1, V_2, \ldots  , V_n$ for the vertices on the left and $N$ unit vectors  $U_1, U_2, \ldots  U_N$ for the vertices on the right such that $\langle V_i, U_j\rangle > m$ if and only if $(i,j)$ is an edge and $\langle V_i, U_j\rangle <- m$ otherwise for some $m>0.$
    
    Prove that there exist $n$ unit vectors in dimension $d$ $V'_1, V'_2, \ldots  , V'_n$ for the vertices on the left and $N$ unit vectors in dimension $d$ $U'_1, U'_2, \ldots  U'_N$ for the vertices on the right such that $\langle V_i, U_j\rangle > m/2$ if and only if $(i,j)$ is an edge and $\langle V_i, U_j\rangle <- m/2$ otherwise, where $d = O(\log(Nn)/m^2).$
    \end{minipage}
    }
    
    \medskip
    
    After it succeeded with this task, we prompted for a variant of Theorem~\ref{thm:main}. 
    
    \medskip
    
    \fbox{
    \begin{minipage}{33em}
    can you attempt to now make a construction in $O(m^{-2}\log n )$ even if $n \ll N$. Perhaps you need to consider the convex program for fixed $V_1, V_2, \ldots , V_n$ that guarantees the existence of $U_1, U_2, .., U_N$ with margin $m.$ You know that this convex program for the given $V_1, V_2, \ldots , V_n$ has a solution. You can reduce the dimension of $V_1, V_2, \ldots , V_n$ via JL or some other projection to $d = O(m^{-2}\log n)$ and then consider the convex program again. Using duals may be useful.
    \end{minipage}
    }
    
    \medskip
    
    This gave us essentially the proof of Theorem~\ref{thm:main} with some minor bugs which we fixed.
    \item Theorem~\ref{thm:mainlowerbound}: We asked ChatGPT to prove a bound of $d = \Omega(k\log n)$ for margin $\Theta(k^{-1/2})$ in the $\ksparsen$ case. We knew that $d = O(k\log n)$ is sufficient in light of our main Theorem~\ref{thm:main}. This time, we generated the prompt with ChatGPT. The final prompt was:
        
    \medskip
    
    \fbox{
    \begin{minipage}{33em}
Problem

Let $N = \binom{n}{k}.$ For each subset $S \subset [n]$ with $|S|=k,$ assign a unit vector $U_S \in \mathbb{S}^{d-1},$ and for each $j \in [n],$ assign a unit vector $V_j \in \mathbb{S}^{d-1},$ such that for some margin $m > 0$:
\begin{align*}
& \langle U_S, V_j \rangle \ge m \quad \text{if } j \in S,
& \langle U_S, V_j \rangle \le -m \quad \text{if } j \notin S.
\end{align*}
Goal

Upper bound $N = \binom{n}{k}$ in terms of $d$ and $m.$

Your goal is to get matching upper and lower bounds for N Currently we have up to multiplicative constants that $d \ge k\log(n)/\log(1/m).$ We have $d$ is at least $k \log n /\log(k)$ and at most $k \log n.$ So there is a gap. Try to fill it in.
    \end{minipage}
    }
        
    \medskip
    
\item Theorem~\ref{thm:main_construction}: Next, we wanted constructions for a large margin when $d = o(k\log(n/k)).$ Interestingly, prompts like the following one (preceded by the conversation for Theorem~\ref{thm:main}) failed:
    
    \medskip
    
    \fbox{
    \begin{minipage}{33em}
        Suppose that I want to achieve a margin $=1/n$ in dimension smaller than $o(k \log n).$ How can I do this?
    \end{minipage}
    }  

    \medskip
    
    ChatGPT would always either give the construction in dimension $2k+1$ which has margin $n^{-\Omega(k)}$ \cite{AlonFranklRodl1985} and Corollary~\ref{cor:marginalgebraic}; or the construction where $V_1, V_2, \ldots, V_n$ are i.i.d. on the sphere which requires dimension $\Omega(k^2\log n)$ \cite{WangEtAl2026}.

    What made ChatGPT succeed is giving it a much more concrete (even if imprecise) goal:
        
    \medskip
    
    \fbox{
    \begin{minipage}{33em}
        Can you try to obtain margin $1/n$ in dimension $k\sqrt{\log n}.$
    \end{minipage}
    }
        
    \medskip

    It produced the current Theorem~\ref{thm:main_construction} in the special case for $d = \Theta(k^2)$ and margin $1/n,$ acknowledging that it only works when $k = o(\sqrt{\log n}).$ Then, we asked ChatGPT to generalize this construction which is a rather mechanical change. 
    
    Interestingly, similar versions of the prompt which are, however, too strong and ask for a potentially incorrect result failed, e.g:
        
    \medskip
    
    \fbox{
    \begin{minipage}{33em}
        Try to find a construction in dimension $O(k)$ with margin $1/n.$
    \end{minipage}
    }
\end{enumerate}